\pdfoutput=1

\documentclass[11pt]{article}

\usepackage[final]{acl}

\usepackage{times}
\usepackage{latexsym}

\usepackage[T1]{fontenc}

\usepackage[utf8]{inputenc}

\usepackage{microtype}

\usepackage{algorithm}
\usepackage{algorithmic}

\usepackage{multirow}
\usepackage{bm}
\usepackage{algorithm}
\usepackage{algorithmic}
\usepackage{booktabs}
\usepackage{amsmath}
\usepackage{amsthm}
\usepackage{amssymb}
\usepackage{mathtools}
\usepackage{tikz}
\usepackage{xcolor}
\usetikzlibrary{arrows}
\usepackage{nicefrac} 

\usepackage{algorithm}
\usepackage{algorithmic}
\usepackage{bm}
\usepackage{subcaption}
\usepackage{xspace}

\newcommand{\vect}[1]{\mathbf{#1}}

\newtheorem{thm}{Theorem}[]
\newtheorem{thmA}{Theorem}[]

\newtheorem{prop}{Proposition}[]
\newtheorem{propA}{Proposition}[]

\DeclareMathOperator*{\argmax}{arg\,max}
\DeclareMathOperator*{\argmin}{arg\,min}

\title{GA-SAM: Gradient-Strength based Adaptive Sharpness-Aware Minimization for Improved Generalization}

\author{Zhiyuan Zhang\textsuperscript{1}, Ruixuan Luo\textsuperscript{2}, Qi Su\textsuperscript{3, 1}, Xu Sun\textsuperscript{1} \\
  \textsuperscript{1}MOE Key Laboratory of Computational Linguistics, School of Computer Science, \\ Peking University\\
  \textsuperscript{2}Center for Data Science, Peking University\\
  \textsuperscript{3}School of Foreign Languages, Peking University\\
   \texttt{\{zzy1210,luoruixuan97,sukia,xusun\}@pku.edu.cn}}
  
\begin{document}
\maketitle
\begin{abstract}
Recently, Sharpness-Aware Minimization (SAM) algorithm has shown state-of-the-art generalization abilities in vision tasks. It demonstrates that flat minima tend to imply better generalization abilities. However, it has some difficulty implying SAM to some natural language tasks, especially to models with drastic gradient changes, such as RNNs. In this work, we analyze the relation between the flatness of the local minimum and its generalization ability from a novel and straightforward theoretical perspective. We propose that the shift of the training and test distributions can be equivalently seen as a virtual parameter corruption or perturbation, which can explain why flat minima that are robust against parameter corruptions or perturbations have better generalization performances. On its basis, we propose a Gradient-Strength based Adaptive Sharpness-Aware Minimization (GA-SAM) algorithm to help to learn algorithms find flat minima that generalize better. Results in various language benchmarks validate the effectiveness of the proposed GA-SAM algorithm on natural language tasks.
\end{abstract}

\section{Introduction}

Recently, researchers~\citep{Adversarial_data_Weight_Perturbation,model_robustness_Against_Weight_Perturbations,parameter_corruption,multi-step-defense,SAM,SAM-LB} propose that for better generalization ability, learning algorithms should find flat minima that have better robustness resistant to parameter corruptions or perturbations. Many learning algorithms that take the flatness or sharpness of the parameters into consideration are motivated by the observation that flat minima tend to imply better generalization abilities. Among them, Sharpness-Aware Minimization (SAM)~\citep{SAM} algorithm has achieved state-of-the-art generalization abilities in vision tasks. It adopts virtual adversarial parameter corruptions or perturbations during training and lowers the risk after parameter corruptions. However, traditional SAM algorithms usually adopt fixed strengths of parameter corruptions and constraint the corruptions with $L_2$-norm or $L_{+\infty}$-norm balls. It cannot conduct flexible strengths of parameter corruptions for different parameters, or during different stages of training. Thus, it is difficult to apply SAM to some natural language tasks, especially to models with drastic gradient changes, such as RNNs. To settle this issue, many adaptive SAM algorithms~\citep{ASAM,SAM-LB} are proposed empirically. In this work, we propose a gradient-strength based adaptive solution based on our theoretical framework.

Existing studies~\citep{Adversarial_data_Weight_Perturbation,multi-step-defense} try to explain the relation between the flatness of the local minimum and its generalization ability according to Probably Approximately Correct (PAC) Bayesian generalization bounds~\citep{Bayes_bound}. In this work, we propose a novel theoretical framework to analyze this relation from a more intuitive and direct perspective. In the Distributionally Robust Optimization (DRO)~\citep{DRO_review} field, the elementary assumption is that there exists a shift between the distributions of the training set and the test set. We propose that a small distribution shift can be equivalently seen as a virtual parameter corruption or perturbation on the loss function. We conduct analytic trials to verify our theoretical account and the results show that it fits the simulation well and can therefore explain why flat minima that are robust against parameter corruptions or perturbations have better generalization performances. We also analyze the strength of the parameter corruption within this framework, based on which we propose a Gradient-Strength based Adaptive Sharpness-Aware Minimization (GA-SAM) algorithm, which can set flexible strengths of parameter corruptions for different parameter groups, during different training stages. 

To validate the effectiveness of the proposed GA-SAM  algorithm, we choose several natural language models and benchmarks, including Convolution Neural Networks (CNN)~\citep{TextCNN} on text classification, Long Short-term Memory (LSTM)~\citep{merityRegOpt} networks on language modeling, and Transformer~\citep{transformer} on neural machine translation. We also compare our proposed GA-SAM algorithm with the traditional SAM algorithm~\citep{SAM} and its multiple variants, including multi-step adversarial parameter defense algorithm~\citep{multi-step-defense}, adaptive SAM~\citep{ASAM}, layer-wise SAM~\citep{SAM-LB} and other possible variants of our proposed algorithm. Experimental results show that our proposed GA-SAM gains better generalization compared to the traditional SAM algorithm and other variants.

Our contributions can be summarized as follows:
\begin{itemize}
    \item We propose a novel theoretical framework to analyze the relation between the flatness of the local minimum and its generalization ability. Under our proposed theoretical framework, the shift of the training and test distributions can be equivalently seen as a virtual parameter corruption or perturbation. Thus, the flatness or the robustness against parameter corruptions can indicate the generalization ability.
    \item On the basis of our novel framework, we further propose a Gradient-Strength based Adaptive Sharpness-Aware Minimization (GA-SAM) algorithm to set flexible strengths of parameter corruptions for different parameter groups, during different stages of training for an improvement over generalization ability. 
    \item Experimental results show the effectiveness of the GA-SAM algorithm compared to the traditional SAM algorithm and its variants.
\end{itemize}

\section{Proposed Theoretical Framework}

In this section, we propose a novel theoretical framework to reveal the relation between distribution shifts and parameter corruptions from an intuitive and direct theoretical perspective.

\subsection{Preliminary}

Let us consider a neural network with the parameter vector $\vect{w}\in \mathbb{R}^n$. For a data instance $\vect{z}=(\vect{x},y)$, denote $\ell(\vect{w};\vect{z})$ as the loss of the data instance, $\mathcal{L}(\vect{w};\mathcal{D})$ as the average loss of a dataset $\mathcal{D}$, and $p(\vect{z})$ as the probability distribution of $\mathcal{D}$, we have:
\begin{align}
\mathcal{L}(\vect{w};\mathcal{D})=\mathbb{E}_{\vect{z}\sim p(\vect{z})}[\ell]=\int_{\vect{z}}p(\vect{z})\ell(\vect{w};\vect{z})d\vect{z}.
\end{align}

Denote $\bm\theta$ as the optimal parameter:
\begin{align}
\bm\theta=\argmin\limits_{\vect{w}}\mathcal{L}(\vect{w};\mathcal{D}),
\end{align}
and the Hessian matrix on $\bm\theta$ is ${\bm H}=\nabla^2_{\bm\theta}\mathcal{L}(\bm\theta;\mathcal{D})$.

Similarly, denote $\mathcal{D}^*$ and $p^*(\vect{z})$ as the test set and its distribution, $\bm\theta^*$ as its optimal parameter, and the Hessian matrix on $\bm\theta^*$ is ${\bm H}^*$. Define the parameter shift of the test and training minima as $\bm\delta=\bm\theta^*-\bm\theta$.

Suppose $n$ parameters are divided into $l$ groups and the $i$-th group has $n_{(i)}$ parameters (\textit{e.g.}, $l=1, n=n_{(1)}$ when the whole model adopt the same strength and we call it model-wise, $n=l, n_{(i)}=1$ when element-wise, $l$ is the layer number when layer-wise, $l$ is the filter number when filter-wise, \textit{etc.}), $\vect{w}=[\vect{w}_{(1)}^\text{T}, \cdots, \vect{w}_{(l)}^\text{T}]^\text{T}$ and $\vect{g}=\nabla_\vect{w}\mathcal{L}(\vect{w};\mathcal{D})=[\vect{g}_{(1)}^\text{T}, \cdots, \vect{g}_{(l)}^\text{T}]^\text{T}$, and $\bm\delta=[\bm\delta_{(1)}^\text{T}, \bm\delta_{(2)}^\text{T},\cdots, \bm\delta_{(i)}^\text{T}, \cdots, \bm\delta_{(l)}^\text{T}]^\text{T}$.

\subsection{The Distribution Shift between the Training and Test Sets}

The elementary assumption in the Distributionally Robust Optimization (DRO)~\citep{DRO_review} field is that there exists a small distributional shift between the training and test sets. Previous studies usually assume that the divergence or the distance of the training and test distributions is bounded by a constant, \textit{e.g.}, the Kullback-Leibler divergence~\citep{KL-kullback1951information}, $\text{KL}(p(\vect{z})||p^*(\vect{z}))\le$ Constant. In this work, more generally, we assume that the $f$-divergence~\citep{measures_entropy} $D_f$ of the distributions is bounded by $C_f$:
\begin{align}
    D_f\left(p^*(\vect{z})||p(\vect{z})\right)=\int_{\vect{z}}p(\vect{z})f\big(\frac{p^*(\vect{z})}{p(\vect{z})}\big)\le C_f,
\end{align}
where the function $f$ is convex and $f(1)=0$, its Taylor expansion should also satisfy  $f(1+x)=a_1x+a_2x^2+o(x^2), a_2\ne 0$. For example, for the Kullback-Leibler divergence (KL-div), $f(1+x)=(1+x)\log(1+x)=x+{x^2}/{2}+o(x^2)$.

\subsection{Parameter Corruptions as Results of Distribution Shifts}
\label{sec:Parameter_Corruptions}

We propose a novel theoretical framework to analyze this relation from a more intuitive and direct perspective. The main theoretical motivation is Theorem~\ref{thm:shift}. Proofs and details are in Appendix. 
\begin{thm}
The distribution shifts of datasets $\mathcal{D}$ and $\mathcal{D}^*$ can be equivalently treated as a parameter corruption near the corresponding minima,
\begin{align}
    \mathcal{L}(\bm\theta^*+\vect{v};\mathcal{D}^*)\approx \mathcal{L}(\bm\theta+\vect{v};\mathcal{D})+\text{Constant},
\end{align}
where Constant=$\mathcal{L}(\bm\theta^*;\mathcal{D}^*)-\mathcal{L}(\bm\theta;\mathcal{D})$. Let $\vect{a}=-\bm\delta$, when $\vect{w}$ is near $\bm\theta$ and $\bm\theta^*$, we have
\begin{align}
\mathcal{L}(\vect{w};\mathcal{D}^*)\approx \mathcal{L}(\vect{w}+\vect{a};\mathcal{D})+\text{Constant}.
\end{align}
\label{thm:shift}
\end{thm}

It shows that the distribution shifts will cause a parameter corruption or parameter shift. Therefore, optimizing the parameter corruption risk $\mathcal{L}(\vect{w}+\vect{a};\mathcal{D})$ can help optimize the loss on the test set $\mathcal{L}(\vect{w};\mathcal{D}^*)$. 
Define $S$ as the possible corruption constraint set of potential corruptions $\vect{a}$. Since $\vect{a}=-\bm\delta$ is determined by the invisible distribution shifts, we optimize the risk under potential corruptions $\vect{a}$ instead, which is exactly the SAM optimization,
\begin{align}
    \bm\theta_\text{SAM} = \argmin\limits_{\vect{w}}\max\limits_{\vect{a}\in S}\mathcal{L}(\vect{w}+\vect{a};\mathcal{D}).
\end{align}

Thus, we reveal why flat minima that are robust against potential parameter corruptions or perturbations have better generalization performances in our theoretical framework. Traditional SAM algorithms adopt fixed strengths of parameter corruptions and constraint the corruptions with $L_2$-norm or $L_{+\infty}$-norm balls, namely $S=\{\vect{a}:\|\vect{a}\|_2\le\epsilon\}$ or $S=\{\vect{a}:\|\vect{a}\|_{+\infty}\le\epsilon\}$. However, in Theorem~\ref{thm:delta}, it reveals that the potential parameter corruption $\bm\delta$ is determined by the distribution shifts and the local geometry near the local minimum in the loss basin. Based on this, we have Proposition~\ref{prop:norm}.
\begin{thm}
Define $r(\vect{z})={{p}^*(\vect{z})}/{{p}(\vect{z})}-1$. When the distribution shift is small enough, namely $r(\vect{z})$ is small, we can estimate the parameter shift $\bm\delta$ as,
\begin{align}
\bm\delta = -{\bm H}^{-1}\mathbb{E}_{p}[r(\vect{z})\nabla_{\bm\theta}\ell(\bm\theta;\vect{z})]+o(\|\bm\delta\|_2).
\end{align}
\label{thm:delta}
\end{thm}
\begin{prop}
Suppose the loss $\mathcal{L}(\vect{w};\mathcal{D})$ is $\mu$-strongly convex\footnote{Note that $\mathcal{L}$ is only required to be $\mu$-strongly convex in the neighborhood of the loss basin including $\bm\theta$ and $\bm\theta^*$, instead of the entire $\mathbb{R}^n$.}, and $D_f(p^*||p)\le C_f$, there exists
\begin{align} 
C_{\bm\delta}=\frac{1+o(1)}{\mu}\sqrt{\frac{C_f}{a_2}\mathbb{E}_{p(\vect{z})}\big[\|\nabla_{\bm\theta}\ell(\bm\theta;\vect{z})\|_2^2\big]}
\end{align}
such that $\|\bm\delta\|_2\le C_{\bm\delta}$, namely $C_{\bm\delta}$ is a upper bound.
\label{prop:norm}
\end{prop}

\begin{figure}[!t]
    \includegraphics[width=0.95\linewidth]{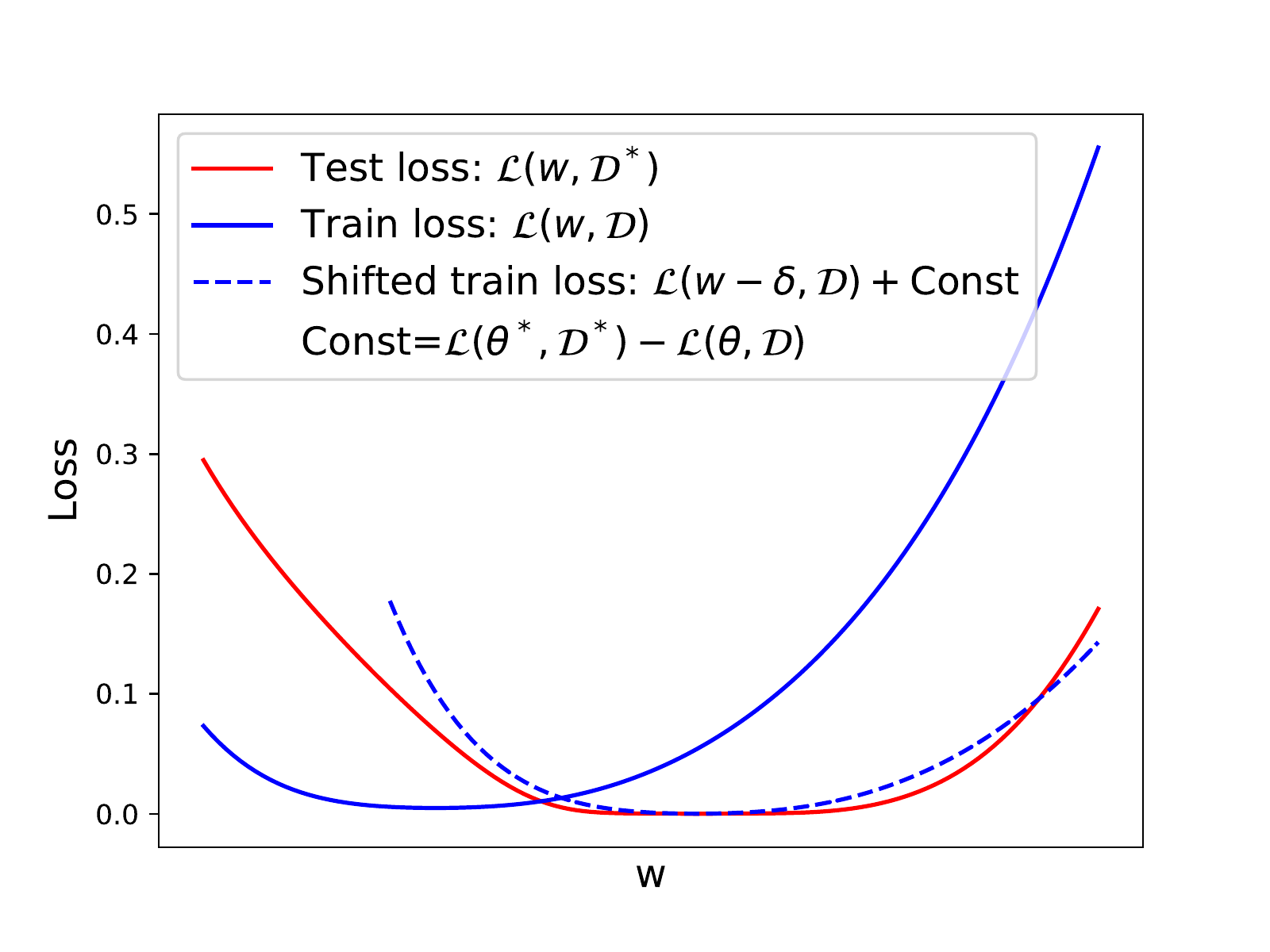}
    \caption{Visualizations of the training loss, test loss, and the shifted training loss. The shifted training loss is similar to the test loss near the local minimum.} 
    \label{fig:shift}
\end{figure}

\subsection{Verification of the Theoretical Framework}

In this section, we conduct several analytic trials\footnote{The details of trials are in Appendix.} to verify our theoretical explanations and analyze the corruption strength. Results show that our theoretical framework fits the simulation.

\textbf{Visualization of Distribution Shift Effects.}
Fig.~\ref{fig:shift} visualizes the training loss, test loss and the shifted training loss. The shifted training loss is similar to the test loss near the local minimum, namely we have $\mathcal{L}(\vect{w};\mathcal{D}^*)\approx\mathcal{L}(\vect{w}-\bm\delta;\mathcal{D})+$Constant, which validates Theorem~\ref{thm:shift}. This phenomenon can also be observed in visualizations of training and test loss landscapes in other studies.

\textbf{Relation between Corruption Strength and Distribution Shift Strength.}
We conduct analytic trials to reveal the relation between the corruption strength and the distribution shift to verify our theoretical framework. Suppose $\mathcal{D}$ is the training set and $\mathcal{D}^*$ is the test set. We can construct a mixed dataset $\mathcal{D}^\text{mix}$, mixed with $(1-\eta)$ of the training data from $\mathcal{D}$ and $\eta$ of the test data from $\mathcal{D}^*$. We have $p^\text{mix}=(1-\eta)p+\eta p^*$ approximately. Define $\eta$ as the relatively distribution shift strength between $\mathcal{D}^\text{mix}$ and $\mathcal{D}$. Proposition~\ref{prop:data} reveals that the corruption strength is proportional to the distribution shift strength, which fits both the simulation results of analytic trails in Fig.~\ref{fig:data} and the intuition.

\begin{prop}
\label{prop:data}
Suppose the mixed distribution of $\mathcal{D}^\text{mix}$ is $p^\text{mix}=(1-\eta)p+\eta p^*$, then we have $D_f(p^\text{mix}||p)\le C_f^\text{mix}=\eta^2 C_f$. Denote $\bm\theta^\text{mix}$ as the optimal parameter on $\mathcal{D}^\text{mix}$, $\bm\delta^\text{mix}=\bm\theta^\text{mix}-\bm\theta$, then we have:
\begin{align}
\frac{\|\bm\delta^\text{mix}\|_2}{\|\bm\delta\|_2}=\frac{C_{\bm\delta}^\text{mix}}{C_{\bm\delta}}=\eta+o(1).
\end{align}
\end{prop}

\begin{figure}[!t]
    \includegraphics[width=0.95\linewidth]{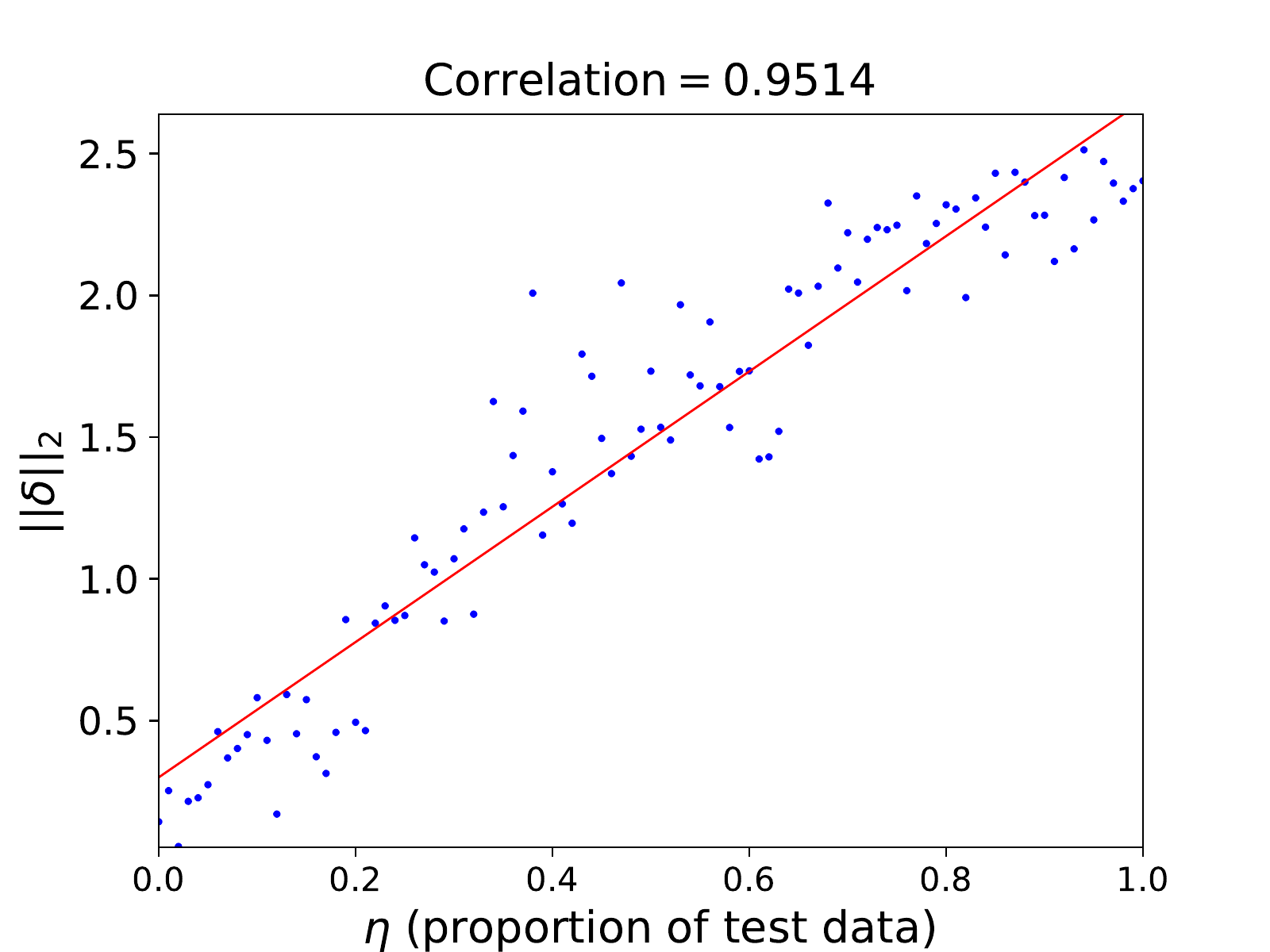}
    \caption{Results of $\|\bm\delta\|_2$ of 100 trials with the same training set $\mathcal{D}$ and different mixed test sets $\mathcal{D}^\text{mix}$ (mixed with $(1-\eta)$ of the training data from $\mathcal{D}$ and $\eta$ of the test data from $\mathcal{D}^*$). $\eta$ can be utilized to measure the strength of the distribution shift between $\mathcal{D}$ and $\mathcal{D}^\text{mix}$. Results show that there exists an approximately linear relationship between $\|\bm\delta\|_2$ and distribution shift strengths.} \label{fig:data}
\end{figure}

\section{Gradient-Strength based Adaptive Sharpness-Aware Minimization}

In this section, we propose a Gradient-Strength based Adaptive Sharpness-Aware Minimization (GA-SAM) algorithm based on the proposed theoretical framework, which conducts flexible strengths of parameter corruptions for different parameter groups for better generalization abilities.

\subsection{Adaptive Sharpness-Aware Minimization}

As illustrated in Section~\ref{sec:Parameter_Corruptions}, the SAM~\citep{SAM} optimization objective is exactly the risk under potential corruption as a result of distribution shift,
\begin{align}
    \bm\theta = \argmin\limits_{\vect{w}}\max\limits_{\vect{a}\in S}\mathcal{L}(\vect{w}+\vect{a};\mathcal{D}),
\end{align}
where the constraint $S$ is $S=\{\vect{a}:\|\vect{a}\|_p\le\epsilon\}$. 

Adaptive SAM algorithms~\citep{ASAM,SAM-LB} are proposed to set flexible strengths of parameter corruptions, which set the constraint $S$ as $S=\{\vect{a}:\|{\bm T}^{-1}\vect{a}\|_p\le\epsilon\}$, where ${\bm T}$ is the transformation matrix (usually diagonal) controlling the corruption strengths of corresponding parameter groups. Define ${\bm T}=\text{diag}\{T_{(1)}{\bm I}_{n_{(1)}}, \cdots, T_{(l)}{\bm I}_{n_{(l)}}\}$, where $T_{(i)}$ controls the corruption strengths of group $i$. Adaptive SAM (ASAM)~\citep{ASAM} empirically adopts $T_{(i)}=\|\vect{w}_{(i)}\|_2$ element-wisely or filter-wisely in CNNs, and layer-wise SAM (Layer-SAM)~\citep{SAM-LB} empirically adopts $T_{(i)}={\|\vect{w}_{(i)}\|_2}/{\|\vect{g}_{(i)}\|_2}$ layer-wisely.

\subsection{The Relation between the Corruption Strength and the Gradient Strength}

We hope to derive $\bm T$ from the theoretical framework instead of setting $\bm T$ empirically. 

In Theorem~\ref{thm:delta}, besides the term $r(\vect{z})$ determined by the distribution shift, corruption strength is also determined by the local geometry near the local minimum (${\bm H}, \nabla_{\bm\theta}\ell(\theta;\vect{z})$). Suppose $n$ is the parameter number,  $G=\mathbb{E}_{p(\vect{z})}[\|\nabla_{\bm\theta}\ell(\bm\theta;\vect{z})\|_2]/\sqrt{n}$ is the average gradient strength. Suppose the Fisher information matrix assumption~\citep{fisher_gradient} holds, namely $\bm H=\mathbb{E}_{p(\vect{z})}[\nabla_{\bm\theta}\ell(\bm\theta;\vect{z})\nabla_{\bm\theta}\ell(\bm\theta;\vect{z})^\text{T}]$, the local geometry is determined by the gradient strength $G$. 

In Proposition~\ref{prop:gradient}, we analyze the relation between the corruption strength and the gradient strength. We have $\|\bm\delta\|_2\propto {\sqrt{n}}/{G}$ and $\|\bm\delta_{(i)}\|_2 \propto {\sqrt{n_{(i)}}}/{G_{(i)}}$, corruption strengths and scales $\bm T$ should be determined by gradient strengths $G$.

\begin{prop}
\label{prop:gradient}
Define the average gradient strength as $G=\mathbb{E}_{p(\vect{z})}[\|\nabla_{\bm\theta}\ell(\bm\theta;\vect{z})\|_2]/\sqrt{n}$, and the average gradient strength of group $i$ as $G_{(i)}=\mathbb\mathbb{E}_{p(\vect{z})}[\|\nabla_{\bm\theta_{(i)}}\ell(\bm\theta;\vect{z})\|_2]/\sqrt{n_{(i)}}$, then 
\begin{align}
\|\bm\delta\|_2\propto \frac{\sqrt{n}}{G},\quad \|\bm\delta_{(i)}\|_2 \propto \frac{\sqrt{n_{(i)}}}{G_{(i)}}.
\end{align}
\end{prop}

\subsection{Proposed Algorithm}

Our proposed algorithm adopts the constraint $S=\{\vect{a}:\|{\bm T}^{-1}\vect{a}\|_p\le\epsilon\}$ in SAM learning, where $\bm T$ are adaptive scales derived from the theoretical framework, and we adopt the multi-step implementation.

In Proposition~\ref{prop:gradient}, we have $\|\bm\delta_{(i)}\|_2\propto {\sqrt{n_{(i)}}}/{G_{(i)}}$. To make the scale of the virtual corruptions $\vect{a}_{(i)}$ propotional to the scale of $\bm\delta_{(i)}$, we should ensure that $\|\vect{a}_{(i)}\|_2\propto\|\bm\delta_{(i)}\|_2$. Suppose $a={\|\vect{a}_{(i)}\|_2}/{\sqrt{n_{(i)}}}$ is the average scale of $\vect{a}_{(i)}$, we have $a \propto T_{(i)}\epsilon$, we should ensure that
\begin{align}
\|\vect{a}_{(i)}\|_2=a\sqrt{n_{(i)}}\propto\|\bm\delta_{(i)}\|_2\propto\frac{\sqrt{n_{(i)}}}{G_{(i)}},
\end{align}
and we can set $T_{(i)}\epsilon\sqrt{n_{(i)}}\propto{\sqrt{n_{(i)}}}/{G_{(i)}}$, therefore $T_{(i)}\propto{1}/{G_{(i)}}={\sqrt{n_{(i)}}}/\mathbb{E}_{p(\vect{z})}[\|\nabla_{\bm\theta_{(i)}}\ell(\bm\theta;\vect{z})\|_2]$. We use $\|\vect{g}_{(i)}\|_2$ to replace $\mathbb{E}_{p(\vect{z})}[\|\nabla_{\bm\theta_{(i)}}\ell(\bm\theta;\vect{z})\|_2]$, then our proposal is derived,
\begin{align}
T_{(i)}=\frac{\sqrt{n_{(i)}}}{\|\vect{g}_{(i)}\|_2\sqrt{n}}.
\end{align}

\begin{table*}[ht]
\centering
\setlength{\tabcolsep}{3pt}
\renewcommand{\baselinestretch}{1.2}
\small
\begin{tabular}{@{}cccccccc@{}}
\toprule
   \multirow{2}*{\textbf{Dataset}} &
  \multicolumn{2}{c}{\textbf{Approach}} & & \multicolumn{1}{c}{\textbf{IMDB (ACC)} } &   \multicolumn{1}{c}{\textbf{PTB-LM (PPL)}} & \multicolumn{1}{c}{\textbf{En-Vi (BLEU)} } & \multicolumn{1}{c}{\textbf{De-En (BLEU)} } \\
 \cmidrule{2-3}\cmidrule{5-8}
  & \multicolumn{2}{c}{\textbf{Base Model}} & & CNN & LSTM & Transformer & Transformer \\
  \midrule
  \textbf{Baseline} & w/o SAM & & & 84.42$\pm$0.12 & 86.70$\pm$0.54 & 30.60$\pm$0.21 & 35.41$\pm$0.13\\
  \midrule
  \multirow{3}*{{\bf Single-step}} & SAM &  & &  84.75$\pm$0.31 (+0.33) &  89.66$\pm$0.25 (+2.96) & 30.79$\pm$0.15 (+0.19) & 35.61$\pm$0.18 (+0.20)\\
 \cmidrule{2-3}\cmidrule{5-8}
  & ASAM & & & 85.05$\pm$0.22 (+0.63) & 90.08$\pm$0.24 (+3.38) & 30.81$\pm$0.24 (+0.21) & 35.56$\pm$0.17 (+0.15) \\
 \cmidrule{2-3}\cmidrule{5-8}
  & Layer-SAM & & & 85.27$\pm$0.19 (+0.83) & 89.82$\pm$0.10 (+3.12) & 30.70$\pm$0.27(+0.10) & 35.78$\pm$0.08 (+0.37) \\
  \midrule
  \multirow{2}*{{\bf Multi-step}}& Multi-step Defense & & & 84.87$\pm$0.15 (+0.45) & 84.74$\pm$0.42 (-1.96) & 30.95$\pm$0.12 (+0.35) & 35.86$\pm$0.13 (+0.45)\\
 \cmidrule{2-3}\cmidrule{5-8}
  & \bf Proposed GA-SAM & & & \textbf{86.11$\pm$0.22} (\textbf{+1.69}) & \textbf{84.52$\pm$0.26} (\textbf{-2.18}) & \textbf{31.15$\pm$0.29} (\textbf{+0.55}) & \textbf{35.95$\pm$0.15} (\textbf{+0.54})\\
\bottomrule
\end{tabular}
\caption{Results of baselines and different SAM algorithms. Results show the effectiveness of GA-SAM.}
\label{tab:results}
\end{table*}

Existing algorithms are mainly single-step based, we adopt the multi-step implementation inspired by multi-step adversarial parameter defense algorithm~\citep{multi-step-defense}, which optimizes
\begin{align}
    \bm\theta=\argmin\limits_{\vect{w}}\mathbb{E}_{\mathcal{B}}\left[\sum\limits_{k=0}^K\frac{\mathcal{L}(\vect{w}+\vect{a}_k;\mathcal{B})}{K+1}\right],
\label{eq:target}
\end{align}
where $\vect{a}_0=\vect{0}$. Suppose the $k$-th update of the corruption is $\vect{u}_k=\argmax_{\|{\bm T}^{-1}\vect{u}\|_p\le \eta}\vect{u}^{\text{T}}\vect{g}_{k-1}$, which is generated based on $\vect{a}_{k-1}$ and $\vect{g}_{k-1}=\nabla_\vect{w}\mathcal{L}(\vect{w}+\vect{a}_{k-1};\mathcal{B})$, where $\eta$ is the step size and following ~\citet{multi-step-defense}, we set $\eta={1.5\epsilon}/{K}$, then
\begin{align}
\vect{u}_{k}=\frac{\eta\big({\bm T}\text{sgn}(\vect{g}_{k-1})\big)\odot|{\bm T}\vect{g}_{k-1}|^\frac{1}{p-1}}{\||{\bm T}\vect{g}_{k-1}|^\frac{1}{p-1}\|_p}.
\end{align}

To get the $k$-th corruption $\vect{a}_k$, we project the updated corruption $\vect{a}_{k-1}+\vect{u}_k$ into the set $S$, 
\begin{align}
\vect{a}_{k}=\Pi_{S}(\vect{a}_{k-1}+\vect{u}_k),
\label{eq:ak}
\end{align}
and the solutions to the commonly adopted $L_2$-norm and $L_{+\infty}$-norm constraints are:
\begin{align}
\Pi_{\|{\bm T}^{-1}\vect{v}\|_2\le\epsilon}{(\vect{v})} =& \frac{\min\{\|{\bm T}^{-1}\vect{v}\|_2,\epsilon\}\vect{v}}{\|{\bm T}^{-1}\vect{v}\|_2};\\
\Pi_{\|{\bm T}^{-1}\vect{v}\|_{+\infty}\le\epsilon}{(\vect{v})} &= {\bm T}\text{clip}({\bm T}^{-1}\vect{v},-\epsilon, \epsilon).
\end{align}

To summarize, we adopt the adaptive scale $T_{(i)}$ according to the gradient strength and a multi-step implementation. It should be noted that even when $K=1$, our multi-step implementation, which optimizes $(\mathcal{L}(\vect{w})+\mathcal{L}(\vect{w}+\vect{a}_1))/2$, is different from the single-step SAM implementation, which optimizes $\mathcal{L}(\vect{w}+\vect{a}_1)$. However, when $K=1$, they have similar speeds since they both require to generate $\vect{a}_1$ and need two backward propagation processes. Our multi-step implementation, however, allows setting a larger $K$ for better generalization. 

We name the proposed algorithm as the Gradient-Strength based Adaptive SAM (\textbf{GA-SAM}). The algorithm is shown in Algorithm~\ref{alg:GASAM}. The proofs and theoretical details are in Appendix.

\begin{algorithm}[!t]
   \caption{GA-SAM Algorithm}
   \label{alg:GASAM}
\begin{algorithmic}[1]
    \REQUIRE Parameters $\vect{w}$; loss $\mathcal{L}$ and dataset $\mathcal{D}$; steps $K$; training iterations; batch size $|\mathcal{B}|$.
    \STATE Prepare batches $\{\mathcal{B}\}$ and initialize $\vect{w}$.
    \STATE Calculate ${\bm T}$ with $T_{(i)}=\frac{\sqrt{n_{(i)}}}{\|\vect{g}_{(i)}\|_2\sqrt{n}}$.
    \WHILE {Training}
    \STATE $\vect{a}_0\gets \vect{0}$.
    \STATE  Calculate the initial loss: $\mathcal{L}(\vect{w};\mathcal{B})$.
    \FOR {$k = 1$ to $K$}
    \STATE Get $\vect{u}_{k}=\frac{\eta\big({\bm T}\text{sgn}(\vect{g}_{k-1})\big)\odot|{\bm T}\vect{g}_{k-1}|^\frac{1}{p-1}}{\||{\bm T}\vect{g}_{k-1}|^\frac{1}{p-1}\|_p}$.
    \STATE Get $\vect{a}_{k}\gets \Pi_S(\vect{a}_{k-1}+\vect{u}_k)$ as Eq. (\ref{eq:ak}).
    \STATE  Calculate the risk: $\mathcal{L}(\vect{w}+\vect{a}_{k};\mathcal{B})$.
    \ENDFOR
    \STATE Update $\vect{w}$ as minimizing Eq. (\ref{eq:target}).
    \ENDWHILE
\end{algorithmic}
\end{algorithm}

\section{Experiments}

In this section, we report the tasks, datasets, and implementation details. Main results are in Table~\ref{tab:results}.

\subsection{Tasks and Datasets}

We adopt three typical neural networks and natural language tasks to validate the effectiveness of the proposed GA-SAM algorithm on NLP tasks. 

On the \textbf{text classification} task, we adopt Convolution Neural Networks (\textbf{CNN})~\citep{TextCNN} on the IMDb movie reviews dataset (\textbf{IMDB})~\citep{IMDB} with the accuracy (\textbf{ACC}) evaluation metric. On the \textbf{language modeling} task, we adopt Long Short-term Memory (\textbf{LSTM})~\citep{merityRegOpt} networks on the English Penn TreeBank (\textbf{PTB-LM})~\citep{PTB-LM} dataset with the perplexity (\textbf{PPL}) evaluation metric. On the neural machine \textbf{translation} task, we adopt the \textbf{Transformer}~\citep{transformer} model based on the Fairseq implementation~\citep{fairseq} on IWSLT 15 English-Vietnamese (\textbf{En-Vi})~\citep{2015iwslt} and IWSLT 14 German-English (\textbf{De-En})~\citep{2014iwslt} datasets with the \textbf{BLEU} score evaluation metric. Compared with the classification and language modeling tasks, the datasets of the machine translation task are relatively large. Other details are in Appendix.

\subsection{Implementations}

We implement our proposed GA-SAM algorithm with a multi-step risk minimization, and set layer-wise adaptive scales $T_{(i)}={\sqrt{n_{(i)}}}/{(\|\vect{g}_{(i)}\|_2\sqrt{n})}$. We also compare our proposed GA-SAM with other existing algorithms. The traditional SAM algorithm~\citep{SAM} adopts a single-step risk implementation and sets fixed scales. ASAM~\citep{ASAM} adopts a single-step risk implementation and sets element-wise adaptive scales $T_{(i)}=|w_{(i)}|$. Layer-SAM~\citep{SAM-LB} adopts a single-step risk implementation and sets layer-wise adaptive scales $T_{(i)}={\|\vect{w}_{(i)}\|_2}/{\|\vect{g}_{(i)}\|_2}$. The multi-step adversarial parameter defense algorithm~\citep{multi-step-defense} adopts a multi-step risk implementation and sets fixed scales. We also try to combine these techniques and implement other possible variants of GA-SAM, and we conduct an ablation study to compare GA-SAM with these variants. 

For a fair comparison, the settings of different SAM algorithms and variants are the same as the base models except for the SAM settings ($K, \epsilon, L_p, T_{(i)}$). We grid search the optimal SAM hyper-parameters for each algorithm. The details of base models and hyper-parameters are in Appendix.

\subsection{Main Results}
The main results are shown in Table~\ref{tab:results}. Evidently, our proposed GA-SAM leads to significant performance gains over the base models.

Single-step SAM algorithms cannot improve the performance of the LSTM base model on the language modeling task. SAM algorithms that set adaptive scales cannot improve the performance of the traditional SAM algorithm consistently. But it may help deal with drastic changes in gradient scales of different parameters or different learning phases in NLP tasks. 
The multi-step risk minimization generally outperforms SAM and helps improve the stability of learning and the generalization abilities of models.
GA-SAM can both achieve the adaptive scales deduced from the theoretical framework and help improve the stability and the generalization via the multi-step implementation. Thus, GA-SAM outperforms base models and other SAM algorithms.

\begin{table}[!t]
\centering
\setlength{\tabcolsep}{3pt}
\renewcommand{\baselinestretch}{1.2}
\small
\begin{tabular}{@{}ccc@{}}
\toprule
   {\textbf{Approach}} & {\textbf{IMDB (ACC)}} & {\textbf{PTB-LM (PPL)}} \\
   \midrule
   \textbf{Baseline} & 84.42$\pm$0.12 & 86.70$\pm$0.54\\
   \midrule
   \textbf{SAM} & 84.75$\pm$0.31 (+0.33) & 89.66$\pm$0.25 (+2.96) \\
   \midrule
   \textbf{GA-SAM} & \textbf{86.11$\pm$0.22} (\textbf{+1.69}) & \textbf{84.52$\pm$0.26} (\textbf{-2.18}) \\
   \midrule
    \textbf{Single-step}  & 85.86$\pm$0.25 (+1.44) & 89.97$\pm$0.27 (+3.27) \\
   \midrule
     \textbf{Element-wise} & 84.87$\pm$0.14 (+0.45) & 84.69$\pm$0.38 (-2.01)\\
     \textbf{Model-wise} & 85.32$\pm$0.21 (+0.90) & 84.55$\pm$0.39 (-2.15)\\
    \midrule
    \multicolumn{3}{l}{\textbf{Variants with other scales} $T_{(i)}$:} \\
    $1$ & 84.87$\pm$0.15 (+0.45) & 84.74$\pm$0.42 (-1.96)\\
    ${\|\vect{w}_{(i)}\|_2}/{\|\vect{g}_{(i)}\|_2}$ & 85.82$\pm$0.06 (+1.40) & 84.58$\pm$0.27 (-2.12)\\
    $1/{\|\vect{g}_{(i)}\|_2}$ & 85.81$\pm$0.23 (+1.39) & 84.90$\pm$0.32 (-1.80)\\
    ${\|\vect{w}_{(i)}\|_2}/{\sqrt{n_{({i})}}}$ & 85.06$\pm$0.07 (+0.64) & 85.04$\pm$0.15 (-1.66) \\
    $\|\vect{w}_{(i)}\|_2$ & 85.42$\pm$0.10 (+1.00) & 85.03$\pm$0.57 (-1.67) \\
\bottomrule
\end{tabular}
\caption{Results of the ablation study. GA-SAM ($K=1$) is compared to multiple variants. Results show that GA-SAM outperforms other variants, and gradient-strength based adaptive scales usually outperform other scales.}
\label{tab:ablation}
\end{table}

\section{Analysis}

In this section, we first conduct an ablation study and analyze the hyper-parameters. Then we illustrate the sharpness and Hessian spectra with GA-SAM and analyze the difference between CV and NLP learning to explain why NLP tasks need GA-SAM.

\begin{table}[!t]
\centering
\setlength{\tabcolsep}{2pt}
\renewcommand{\baselinestretch}{1.2}
\small
\begin{tabular}{@{}ccc@{}}
\toprule
   {\textbf{Approach}} & {\textbf{IMDB (ACC)}} & {\textbf{PTB-LM (PPL)}} \\
   \midrule
   \textbf{Baseline} & 84.42$\pm$0.12 & 86.70$\pm$0.54\\
   \midrule
    \multicolumn{3}{l}{\textbf{GA-SAM w/ diff. (with different)} $K$:} \\
    K = 1  & \textbf{86.11$\pm$0.22} (\textbf{+1.69}) & 84.52$\pm$0.26 ({-2.18}) \\
    K = 2  & 85.83$\pm$0.23 (+1.41) & 84.48$\pm$0.52 (-2.22) \\
    K = 3  & 84.97$\pm$0.23 (+0.55) & 84.14$\pm$0.29 (-2.56)\\
    K = 4  & 77.55$\pm$3.25 (-6.87) & \textbf{84.05$\pm$0.23} (\textbf{-2.65})\\
    K = 5  & 70.49$\pm$5.10 (-13.9) & 84.59$\pm$0.36 (-2.11)\\
     \midrule
    {\textbf{w/ diff.} $\epsilon\ (L_2)$:} & $\times 10^{-2}$ & $\times 10^{-3}$ \\
    $\epsilon$ = 0.1$\times$  & 84.93$\pm$0.27 (+0.51) & 85.24$\pm$0.40 (-1.46)\\
    $\epsilon$ = 0.5$\times$  & 85.11$\pm$0.22 (+0.69) & 84.90$\pm$0.09 (-1.80)\\
    $\epsilon$ = 1$\times$  & 85.36$\pm$0.26 (+0.94) & 84.91$\pm$0.17 (-1.79)\\
    $\epsilon$ = 5$\times$  & 85.62$\pm$0.70 (+1.20) & \textbf{84.69$\pm$0.40} (\textbf{-2.01})\\
    $\epsilon$ = 10$\times$  & \textbf{85.77$\pm$0.11} (\textbf{+1.35}) & 85.38$\pm$0.41 (-1.32)\\
    $\epsilon$ = 50$\times$  & 64.14$\pm$10.3 (-20.3) & 507.0$\pm$243 (+420)\\
    \midrule
    {\textbf{w/ diff.} $\epsilon\ (L_{+\infty})$:} & $\times 10^{-4}$ & $\times 10^{-5}$  \\
    $\epsilon$ = 0.2$\times$  & 84.91$\pm$0.89 (+0.49) & 85.12$\pm$0.55 (-1.58) \\
    $\epsilon$ = 0.5$\times$  & 85.54$\pm$0.18 (+1.12) & 85.19$\pm$0.86 (-1.51) \\
    $\epsilon$ = 0.8$\times$  & 85.63$\pm$0.15 (+1.21) & \textbf{84.52$\pm$0.26} (\textbf{-2.18}) \\
    $\epsilon$ = 1$\times$  & \textbf{86.11$\pm$0.22} (\textbf{+1.69}) & 85.06$\pm$0.51 (-1.64) \\
    $\epsilon$ = 2$\times$  & 85.74$\pm$0.18 (+1.32) & 85.20$\pm$0.94 (-0.96) \\
    $\epsilon$ = 5$\times$  & 65.96$\pm$2.74 (-18.5) & 85.87$\pm$1.19 (-0.83) \\
\bottomrule
\end{tabular}
\caption{Analysis of hyper-parameters. We implement GA-SAM with different $K$ and $\epsilon$ (under $L_2$ and $L_{+\infty}$).}
\label{tab:hyperparameter}
\end{table}

\subsection{Ablation Study}

We implement the single-step variants, element-wise or model-wise variants, and variants with other scales $T_{(i)}$ on IMDB and PTB-LM. The results of the ablation study are reported in Table~\ref{tab:ablation}. 

Experimental results show that layer-wise implementation outperforms element-wise or model-wise implementations and GA-SAM with the multi-step implementation ($K=1$) outperforms GA-SAM with the single-step implementation.

For adaptive scales, the multi-step adversarial parameter defense algorithm~\citep{multi-step-defense} adopts $T_{(i)}=1$. We also adopt $T_{(i)}={\|\vect{w}_{(i)}\|_2}/{\|\vect{g}_{(i)}\|_2}$ following \citet{SAM-LB}, and $T_{(i)}={\|\vect{w}_{(i)}\|_2}$ following \citet{ASAM}. We also try other variants with similar formulas. Experimental results show that GA-SAM outperforms other variants. Gradient-strength ($\|\vect{g}_{(i)}\|_2$) based adaptive scales usually outperform other adaptive scales, and adaptive scales can enhance the multi-step adversarial parameter defense algorithm generally, which validates our theoretical framework.

\begin{table}[!t]
\centering
\setlength{\tabcolsep}{2pt}
\renewcommand{\baselinestretch}{1.2}
\small
\begin{tabular}{@{}ccc@{}}
\toprule
   {\textbf{Corruption}} & {\textbf{Baseline}} & {\textbf{GA-SAM}} \\
   \midrule
   \textbf{w/o Corruption} & 30.60$\pm$0.21 & 31.15$\pm$0.29\\
     \midrule
    $L_2$, $\epsilon$ = 0.05  & 30.26 (-0.34)  & 30.90 (-0.25)\\
    $L_2$, $\epsilon$ = 0.1  & 29.36 (-1.24)  & 30.26 (-0.89) \\
    $L_2$, $\epsilon$ = 0.2  & 6.46 (-24.14)  & 17.87 (-13.28) \\
    \midrule
    $L_{+\infty}$, $\epsilon$ = 0.0001  & 30.30 (-0.30) & 30.72 (-0.43) \\
    $L_{+\infty}$, $\epsilon$ = 0.0002  & 29.95 (-0.65) & 30.67 (-0.48)\\
    $L_{+\infty}$, $\epsilon$ = 0.0005  & 5.33 (-25.27) & 29.14 (-2.01)  \\
\bottomrule
\end{tabular}
\caption{The parameter robustness of baselines and GA-SAM on En-Vi. Minima with GA-SAM are more robust.}
\label{tab:robustness}
\end{table}

\begin{figure}[!t]
    \includegraphics[width=0.9\linewidth]{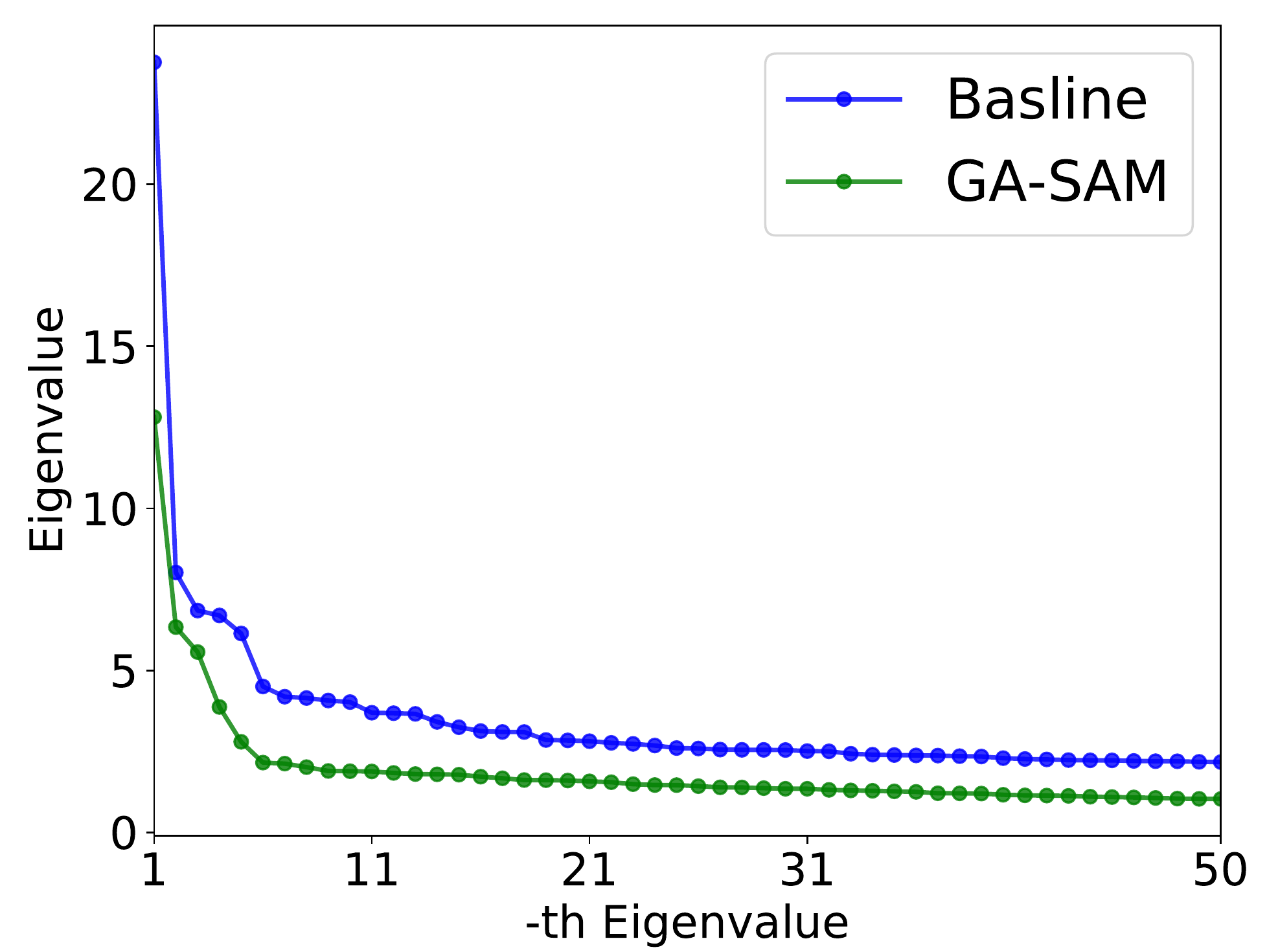}
    \caption{Top 50 eigenvalues of Hessians of baselines and GA-SAM. GA-SAM can help find flat minima.} 
    \label{fig:hessian}
\end{figure}

\subsection{Analysis of Hyper-parameters}

We analyze the influence of hyper-parameters and settings in GA-SAM learning in Table~\ref{tab:hyperparameter}. 

On IMDB, adopting larger $K$ cannot improve the accuracy, while on PTB-LM, larger $K$ can achieve better PPL. The reason might be that on PTB-LM, the LSTM model with more drastic gradient changes needs more steps for better stability.

Both under $L_2$ and $L_{+\infty}$ constraints, on both datasets, the performance can often be improved substantially with small $\epsilon$. However, when $\epsilon$ grows too large, it may harm the learning and the performance will drop. We also find that in this work, the best performance is achieved under $L_{+\infty}$.

\subsection{Sharpness and Parameter Robustness}

As analyzed in our theoretical framework, flat minima tend to imply better generalization abilities. In this section, we validate that GA-SAM can help find flat minima, which tends to help improve the generalization abilities of models. 

The sharpness near the local minima can be evaluated by the parameter robustness against parameter corruptions via the multi-step adversarial parameter corruption algorithm~\citep{multi-step-defense}. In Table~\ref{tab:robustness}, we evaluate the robustness of the baselines and models with GA-SAM against the $L_2$ or $L_{+\infty}$ constrained parameter corruptions on En-Vi. We can see that models with GA-SAM are more robust than baselines, which implies that GA-SAM can help find flat minima. We also adopt the Fisher information matrix assumption~\citep{fisher_gradient} to estimate the top 50 eigenvalues of the Hessian matrix to evaluate the sharpness of minima. As shown in Fig.~\ref{fig:hessian}, the eigenvalues of models with GA-SAM are lower, which illustrates that GA-SAM can help find flat minima.

\begin{table}[!t]
\centering
\setlength{\tabcolsep}{3pt}
\renewcommand{\baselinestretch}{1.2}
\small
\begin{tabular}{@{}ccc@{}}
\toprule
   {\textbf{Approach}} & {\textbf{En-Vi (BLEU)}} & {\textbf{De-En (BLEU)}} \\
   \midrule
   \textbf{Baseline} & 30.60$\pm$0.21 & 35.41$\pm$0.13\\
   \midrule
   \textbf{SAM} & 30.79$\pm$0.15 (+0.19) &35.61$\pm$0.18 (+0.20) \\
   \midrule
    \textbf{FreeLB} & 30.91$\pm$0.09 (+0.31) & 35.49$\pm$0.11 (+0.08)\\
   \midrule
   \textbf{GA-SAM} & \textbf{31.15$\pm$0.29 (+0.55)} & \textbf{35.95$\pm$0.15 (+0.54)} \\
\bottomrule
\end{tabular}
\caption{Comparisons to FreeLB.}
\label{tab:adversarial_examples}
\end{table}

\subsection{Comparisons to Adversarial Training.}

Some adversarial training algorithms improve the generalization ability of neural networks by optimizing the loss of virtual adversarial examples. In this section, we compare GA-SAM with FreeLB~\citep{FreeLB}, an existing high-performance algorithm for NLP tasks, on En-Vi and De-En. Detailed settings are in Appendix.

In Table~\ref{tab:adversarial_examples}, we can see that FreeLB~\citep{FreeLB} can improve the accuracy of NLP models. The reasons that FreeLB~\citep{FreeLB} works may lie in two aspects: (1) FreeLB adopts a multi-step minimization that is helpful for training stability; and (2) FreeLB only involves attacks on word embeddings while the ideal attack strengths for parameters in different layers vary a lot due to gradient vanishing and explosion in NLP models. Therefore, FreeLB does not need flexible scales as necessarily as SAM. However, GA-SAM can still outperform FreeLB.

\subsection{Why NLP Tasks Need GA-SAM}

Compared with the traditional SAM algorithm, GA-SAM adopts the multi-step risk minimization and gradient-strength based adaptive corruption strengths. From the ablation study, we can see that the multi-step risk minimization can enhance the traditional SAM algorithm on NLP tasks. It shows that NLP tasks do need GA-SAM for better stability and generalization with the multi-step risk minimization algorithm. The gradient-strength based adaptive corruption strengths can also enhance SAM algorithms since gradient strengths change drastically during different learning phases and gradient strengths vary in different layers.

\begin{figure}[!t]
    \includegraphics[width=0.95\linewidth]{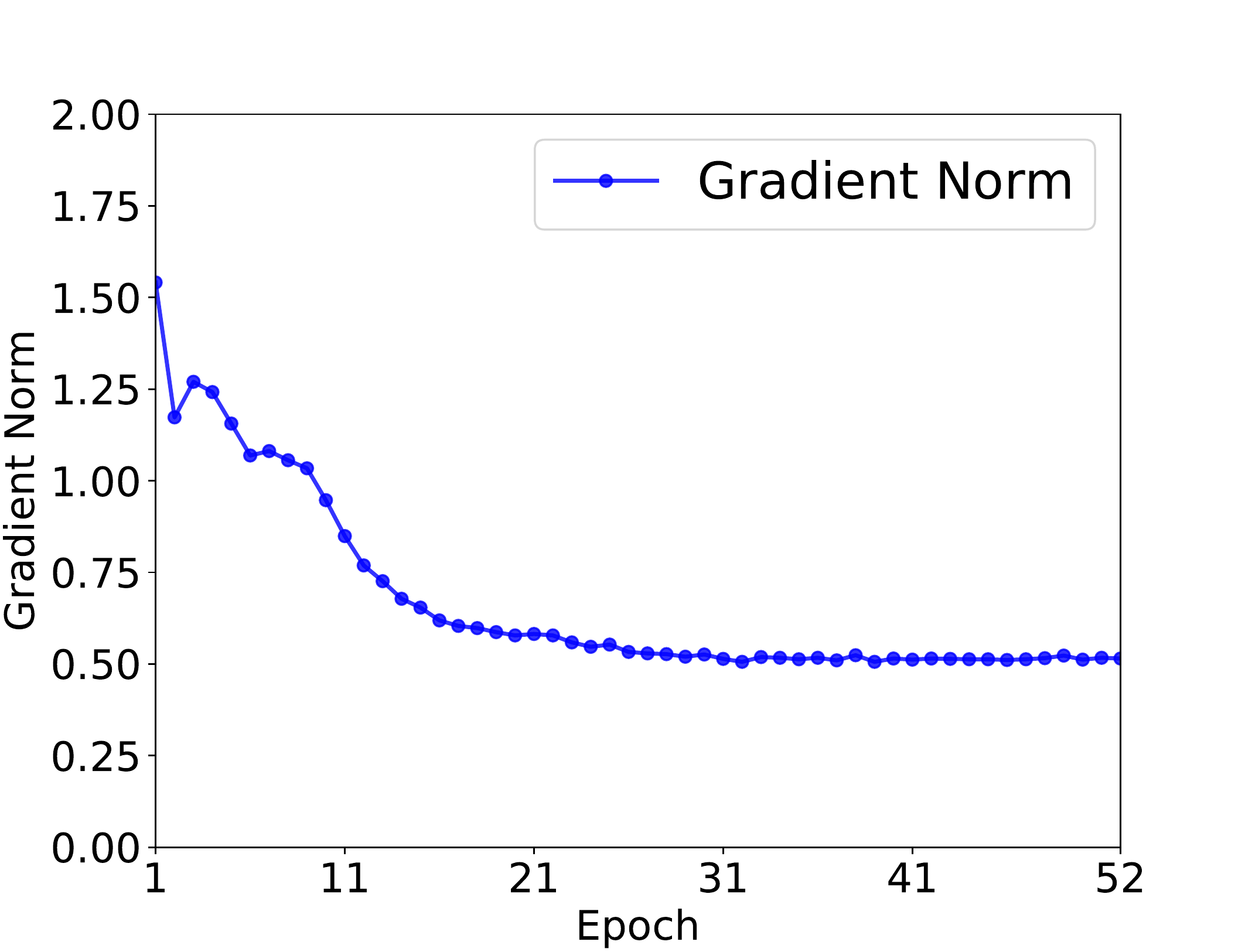}
    \caption{Gradient norms in different learning phases.} \label{fig:epoch}
\end{figure}

\begin{figure}[!t]
    \includegraphics[width=0.95\linewidth]{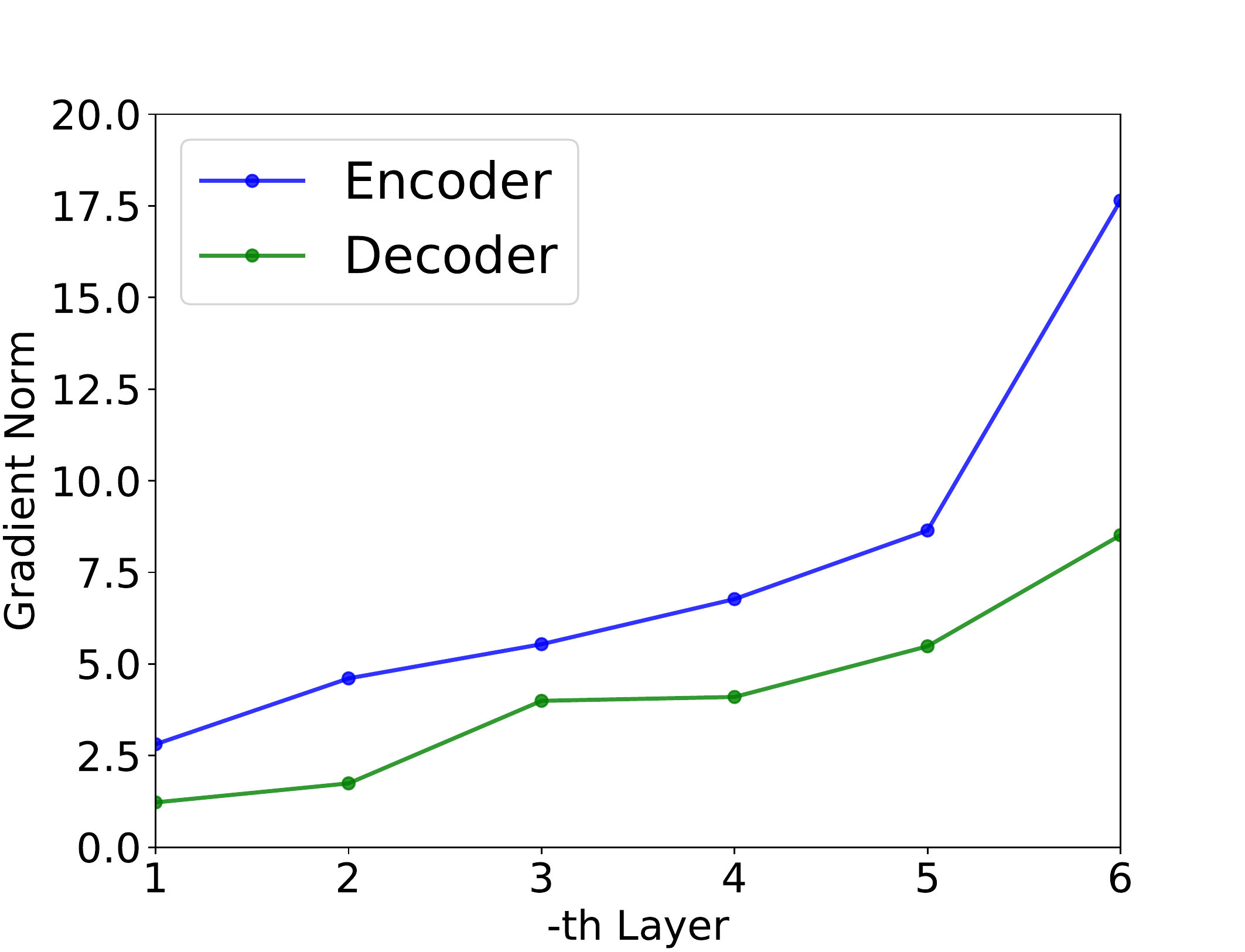}
    \caption{Gradient norms of different layers in Transformer encoder and decoder. 6-th layer is the highest layer in the encoder or decoder, which has the largest gradients.} \label{fig:layer}
\end{figure}

In this section, we further visualize the gradient strengths during different learning phases and different layers on En-Vi to illustrate why NLP tasks need GA-SAM. As shown in Fig.~\ref{fig:epoch}, during the early phase of learning, the gradient norms are much larger and we should conduct parameter corruptions with smaller strengths. As shown in Fig.~\ref{fig:layer}, we can see that higher layers of the Transformer encoder or decoder have larger gradients and need smaller parameter corruptions. To conclude, we need gradient-strength based adaptive corruption strengths since gradient strengths vary in different learning phases and different layers, while the traditional SAM algorithm fails to conduct flexible strengths of parameter corruptions for different parameters, or during different stages of training.

\subsection{Computational Complexity}

Single-step implementations usually involve two forward and backward propagations, including one for generating the corruption $\vect{a}$ and another for optimizing the $\mathcal{L}(\vect{w}+\vect{a})$ loss function or the $\mathcal{L}(\vect{w})+\lambda \vect{a}^\text{T}\nabla_\vect{w}\mathcal{L}(\vect{w})$ loss function). Multi-step implementations involve $K+1$ forward and backward propagations, including $K+1$ times for forwarding and backwarding $\{\mathcal{L}(\vect{w}+\vect{a}_k)\}_{0\le k\le K}$. 

Therefore, when $K=1$, the neural network forward and backward propagation cost (involving two forward and backward propagations) of GA-SAM is the same as SAM or GA-SAM with the single-step implementation. In Table~\ref{tab:ablation}, we conduct a fair comparison between GA-SAM ($K=1$) and SAM or GA-SAM with the single-step implementation (all involve two forward and backward propagations). GA-SAM still outperforms SAM and GA-SAM with the single-step implementation, which illustrates the effectiveness of our proposed GA-SAM and the multi-step implementation.

\section{Related Work}

\subsection{Parameter Corruptions or Perturbations}

Besides adversarial examples~\citep{Intriguing_properties_of_neural_networks,Adversarial_examples_in_the_physical_world,Towards_Evaluating_the_Robustness_of_Neural_Networks} and adversarial training~\citep{Explaining_and_Harnessing_Adversarial_Examples,Unified-min-max,Towards_Deep_Learning_Models_Resistant_to_Adversarial_Attacks,YOPO,FreeLB} concerning adversarial examples, existing studies also concerns small changes on neural network parameters, namely parameter corruptions~\citep{parameter_corruption,multi-step-defense} or perturbations~\citep{Can-AWP-inject-backdoor,Adversarial_data_Weight_Perturbation}. 

Existing studies research parameter corruptions or perturbations mainly for better generalization ability~\citep{Regularizing_NN_via_Adversarial_Perturbation_L2,SAM,ASAM,SAM-LB,efficient-SAM,parameter_corruption,multi-step-defense}, safety issue~\citep{Can-AWP-inject-backdoor,TBT}, analyzing the loss change allocation to parameters~\citep{LCA}, analyzing the compression~\citep{Stronger_Generalization_Compression} or parameter quantization~\citep{Data-Free-Quant,tensorRT,parameter_L1}.

\subsection{Generalization and Flat Minima} 

Existing studies~\citep{Sharp_minima_can_generalize_for_deep_nets,On_Large-Batch_Training,Entropy-SGD,diffusion_flat_minima,Adversarial_data_Weight_Perturbation,parameter_corruption,multi-step-defense} show that local minima that are robust against to parameter corruptions are usually flat minima, which tends to have better generalization ability. A line of Sharpness-aware Minimization (SAM)~\citep{SAM} algorithms drive parameters away from sharp minima via virtual parameter corruptions. Other researches acquire flat minima via adopting adaptive scales of parameter corruptions~\citep{ASAM,SAM-LB}, rescaling parameter corruptions~\citep{SAM-LB,efficient-SAM}, adopting a multi-step implementation~\citep{multi-step-defense} or sharpness-aware learning rates~\citep{SALR}.

\section{Limitation and Broader Impact}

\textbf{Limitation.} One limitation of our work is that, similar to other SAM learning, the hyper-parameters tuning, especially $\epsilon$, involves many numerical experiments, which is time costly and environmentally unfriendly. To settle this issue, we recommend researchers binary search the proper order of magnitude of $\epsilon$ first, and then grid search $\epsilon$ in a small range for a faster hyper-parameter search, instead of directly grid searching $\epsilon$ in a large range.

\textbf{Broader Impact.} Our work is beneficial for the security of NLP models since our work can help improve the robustness of NLP models against parameter corruptions, which can occur as random noises at the hardware level, quantization, or model compression. Our work also has negative social impacts. Our proposed GA-SAM can be utilized to enhance NLP models and improve the accuracy of base models. However, the hyper-parameters tuning, especially $\epsilon$, involves many numerical experiments, which is also a limitation of our work, and it is environmentally unfriendly.

\section{Conclusion}
In this paper, we propose a novel theoretical framework to analyze the relation between parameter corruptions and generalization abilities. Based on our proposed framework, we propose a Gradient-Strength based Adaptive Sharpness-Aware Minimization (GA-SAM) algorithm. Experimental results validate the effectiveness of GA-SAM compared to the traditional SAM algorithm and its variants. Further analyses also show that GA-SAM can help find flat minima and improve the generalization ability of neural networks.

\section*{Acknowledgement}
The authors would like to thank the reviewers for their helpful comments. This work is supported by Natural Science Foundation of China (NSFC) No. 62176002 and Beijing Natural Science Foundation of China (4192057). Xu Sun is the corresponding author.

\bibliography{acl}
\bibliographystyle{acl_natbib}

\appendix
\section{Theoretical Details}

\subsection{Proofs of Theorem~\ref{thmA:shift}}
\begin{thmA}
The distribution shifts of datasets $\mathcal{D}$ and $\mathcal{D}^*$ can be equivalently treated as a parameter corruption near the corresponding minima,
\begin{align}
    \mathcal{L}(\bm\theta^*+\vect{v};\mathcal{D}^*)\approx \mathcal{L}(\bm\theta+\vect{v};\mathcal{D})+\text{Constant},
\end{align}
where Constant=$\mathcal{L}(\bm\theta^*;\mathcal{D}^*)-\mathcal{L}(\bm\theta;\mathcal{D})$. Let $\vect{a}=-\bm\delta$, when $\vect{w}$ is near $\bm\theta$ and $\bm\theta^*$, we have
\begin{align}
\mathcal{L}(\vect{w};\mathcal{D}^*)\approx \mathcal{L}(\vect{w}+\vect{a};\mathcal{D})+\text{Constant}.
\end{align}
\label{thmA:shift}
\end{thmA}

\begin{proof}

Define $f(\vect{v})=\mathcal{L}(\bm\theta^*+\vect{v};\mathcal{D}^*)-\mathcal{L}(\bm\theta+\vect{v};\mathcal{D})$, namely $\mathcal{L}(\bm\theta^*+\vect{v};\mathcal{D}^*)= \mathcal{L}(\bm\theta+\vect{v};\mathcal{D})+f(\vect{v})$. First we prove there exists $C_{H}=\rho C_f^{\frac{1}{2}}+LC_{\bm\delta}=o(1)$ such that $\|\bm H^*-\bm H\|_2\le C_{H}$, where $\rho=\mathbb{E}_{p(\vect{z})}[\|\nabla_{\bm\theta}^2\ell(\bm\theta;\vect{z})\|_2^{2}]^\frac{1}{2}$, and $\|\nabla_{\bm\theta^*}^2\ell(\bm\theta^*;\vect{z})-\nabla_{\bm\theta}^2\ell(\bm\theta;\vect{z})\big\|_2\le L\|\bm\theta^*-\bm\theta\|_2$. We have,

\begin{align}
&\|\bm H^*-\bm H\|_2\\
&=\big\|\int_{\vect{z}}\big\{\big(p^*(\vect{z})-p(\vect{z})\big)\nabla_{\bm\theta}^2\ell(\bm\theta;\vect{z})\\
&+p^*(\vect{z})(\nabla_{\bm\theta^*}^2\ell(\bm\theta^*;\vect{z})-\nabla_{\bm\theta}^2\ell(\bm\theta;\vect{z})\big\}d\vect{z}\big\|_2\\
&\le \big\|\mathbb{E}_{p(\vect{z})}\big[\big(\frac{p^*(\vect{z})}{p(\vect{z})}-1\big)\nabla_{\bm\theta}^2\ell(\bm\theta;\vect{z})\big]\big\|_2\\
&+\big\|\mathbb{E}_{p^*(\vect{z})}\big[\nabla_{\bm\theta^*}^2\ell(\bm\theta^*;\vect{z})-\nabla_{\bm\theta}^2\ell(\bm\theta;\vect{z})\big]\big\|_2\\
&\le \big\|\mathbb{E}_{p(\vect{z})}\big[r(\vect{z})\nabla_{\bm\theta}^2\ell(\bm\theta;\vect{z})\big]\big\|_2+L\|\bm\delta\|_2\\
&\le \mathbb{E}_{p(\vect{z})}[|r(\vect{z})|^{2}]^\frac{1}{2}\mathbb{E}_{p(\vect{z})}\big[\|\nabla_{\bm\theta}^2\ell(\bm\theta;\vect{z})\|_2^{2}\big]^\frac{1}{2}\\
&+LC_{\bm\delta}\\
&=\rho C_f^{\frac{1}{2}}+LC_{\bm\delta}=C_H.
\end{align}

Consider the gradients,
\begin{align}
&\nabla_\vect{v}f(\vect{v})\\
&=\nabla_\vect{v}\mathcal{L}(\bm\theta^*+\vect{v};\mathcal{D}^*)-\nabla_\vect{v}\mathcal{L}(\bm\theta+\vect{v};\mathcal{D})\\
&=(\bm H^*-\bm H)\vect{v}+o(\|\vect{v}\|_2),\\
&\nabla_\vect{v}\mathcal{L}(\bm\theta+\vect{v};\mathcal{D})=\bm H\vect{v}+o(\|\vect{v}\|_2),\\
&\frac{\|\nabla_\vect{v}f(\vect{v})\|_2}{\|\nabla_\vect{v}\mathcal{L}(\bm\theta+\vect{v};\mathcal{D})\|_2}\\
&=\frac{\|(\bm H^*-\bm H)\vect{v}\|_2+o(\|\vect{v}\|_2)}{\|\bm H\vect{v}\|_2+o(\|\vect{v}\|_2)}\\
&\le\frac{C_H+o(1)}{\mu+o(1)}=o(1).
\end{align}

Consider the function $f$, $\nabla_\vect{v}f(\vect{0})=\vect{0}, \nabla^2_\vect{v}f(\vect{0})=\bm H^*-\bm H$,
\begin{align}
&|f(\vect{v})-f(\vect{0})|\\
&=\frac{1}{2}\vect{v}^\text{T}(\nabla^2_\vect{v}f(\vect{0}))\vect{v}+o(\|\vect{v}\|_2^2)\\
&\le \frac{1}{2}C_H\|\vect{v}\|_2^2+o(\|\vect{v}\|_2^2),\\
&|\mathcal{L}(\bm\theta+\vect{v};\mathcal{D})-\mathcal{L}(\bm\theta;\mathcal{D})|\\
&=|\vect{v}^\text{T}\nabla_{\bm\theta}\mathcal{L}(\bm\theta;\mathcal{D})+\frac{1}{2}\vect{v}^\text{T}\bm H\vect{v}|+o(\|\vect{v}\|_2^2)\\
&\ge \frac{1}{2}\mu\|\vect{v}\|_2^2+o(\|\vect{v}\|_2^2),\\
&\frac{|f(\vect{v})-f(\vect{0})|}{|\mathcal{L}(\bm\theta+\vect{v};\mathcal{D})-\mathcal{L}(\bm\theta;\mathcal{D})|}\\
&\le\frac{\frac{1}{2}C_H\|\vect{v}\|_2^2+o(\|\vect{v}\|_2^2)}{\frac{1}{2}\mu \|\vect{v}\|_2^2+o(\|\vect{v}\|_2^2)}\\
&=\frac{C_H+o(1)}{\mu+o(1)}=o(1).
\end{align}

Therefore, the change in the term $f(\vect{v})$ in the loss function can be omitted compared to the change in the loss function,
\begin{align}
    &\mathcal{L}(\bm\theta^*+\vect{v};\mathcal{D}^*)\\
    &= \mathcal{L}(\bm\theta+\vect{v};\mathcal{D})+f(\vect{v})\\
    &\approx \mathcal{L}(\bm\theta+\vect{v};\mathcal{D})+f(\vect{0})\\
    &=\mathcal{L}(\bm\theta+\vect{v};\mathcal{D})+\text{Constant}.
\end{align}

Let $\vect{a}=-\bm\delta, \vect{w}=\bm\theta^*+\vect{v}$, we have,
\begin{align}
 \mathcal{L}(\vect{w};\mathcal{D}^*)\approx
 \mathcal{L}(\vect{w}+\vect{a};\mathcal{D})+\text{Constant}.
\end{align}

\end{proof}

\subsection{Proofs of Theorem~\ref{thmA:delta}}
\begin{thmA}
Define $r(\vect{z})={{p}^*(\vect{z})}/{{p}(\vect{z})}-1$. When the distribution shift is small enough, namely $r(\vect{z})$ is small, we can estimate the parameter shift $\bm\delta$ as,
\begin{align}
\bm\delta = -{\bm H}^{-1}\mathbb{E}_{p}[r(\vect{z})\nabla_{\bm\theta}\ell(\bm\theta;\vect{z})]+o(\|\bm\delta\|_2).
\end{align}
\label{thmA:delta}
\end{thmA}

\begin{proof}
With the change-of-measure technique, we have
\begin{align}
\mathbb{E}_{p(\vect{z})}\big[r(\vect{z})\big]=0.
\end{align}

According to the definition,
\begin{align}
\nabla_{\bm\theta}\mathcal{L}(\bm\theta;\mathcal{D})=\nabla_{\bm\theta^*}\mathcal{L}(\bm\theta^*;\mathcal{D}^*)=\vect{0}.
\end{align}

Conduct Taylor Expansion, we have,
\begin{align}
\vect{0}&=\nabla_{\bm\theta^*}\mathcal{L}(\bm\theta^*;\mathcal{D}^*)-\nabla_{\bm\theta}\mathcal{L}(\bm\theta;\mathcal{D})\\
&=\nabla_{\bm\theta^*}\mathcal{L}(\bm\theta^*;\mathcal{D}^*)-\\
&\big(\nabla_{\bm\theta^*}\mathcal{L}(\bm\theta^*;\mathcal{D})-{\bm H}(-\bm \delta)+o(\|\bm\delta\|_2)\big).
\end{align}

Solve it, we have,
\begin{align}
\bm\delta &= -{\bm H}^{-1}\big(\nabla_{\bm\theta^*}\mathcal{L}(\bm\theta^*;\mathcal{D}^*)\\
&-\nabla_{\bm\theta^*}\mathcal{L}(\bm\theta^*;\mathcal{D})\big)+o(\|\bm\delta\|_2).
\end{align}

Consider $\nabla_{\bm\theta^*}\mathcal{L}(\bm\theta^*;\mathcal{D}^*)-\nabla_{\bm\theta^*}\mathcal{L}(\bm\theta^*;\mathcal{D})$,
\begin{align}
&\nabla_{\bm\theta^*}\mathcal{L}(\bm\theta^*;\mathcal{D}^*)-\nabla_{\bm\theta^*}\mathcal{L}(\bm\theta^*;\mathcal{D})\\
&=\int_{\vect{z}}\big(p^*(\vect{z})-p(\vect{z})\big)\nabla_{\bm\theta^*}\ell(\bm\theta^*;\vect{z})d\vect{z}\\
&=\int_{\vect{z}}\big(\frac{p^*(\vect{z})}{p(\vect{z})}-1\big)p(\vect{z})\nabla_{\bm\theta^*}\ell(\bm\theta^*;\vect{z})d\vect{z}\\
&=\mathbb{E}_{p(\vect{z})}[r(\vect{z})\nabla_{\bm\theta^*}\ell(\bm\theta^*;\vect{z})]\\
&=\mathbb{E}_{p(\vect{z})}[r(\vect{z})\big(\nabla_{\bm\theta}\ell(\bm\theta;\vect{z})+\\
&\nabla_{\bm\theta}^2\ell(\bm\theta;\vect{z})\bm{\delta}+o(\|\bm\delta\|_2)\big)]\\
&=\mathbb{E}_{p(\vect{z})}[r(\vect{z})\nabla_{\bm\theta}\ell(\bm{\theta};\vect{z})]+o(\|\bm\delta\|_2),
\end{align}
where the term $r(\vect{z})\nabla_{\bm\theta}^2\ell(\bm\theta;\vect{z})\bm{\delta}=o(\|\bm\delta\|_2)$ since the distribution shift $r(\vect{z})=o(1)$. To conclude, 
\begin{align}
\bm\delta &= -{\bm H}^{-1}\mathbb{E}_{p}[r(\vect{z})\nabla_{\bm\theta}\ell(\bm\theta;\vect{z})]+o(\|\bm\delta\|_2).
\end{align}
\end{proof}

\subsection{Proofs of Propositions}
\begin{propA}
Suppose the loss $\mathcal{L}(\vect{w};\mathcal{D})$ is $\mu$-strongly convex\footnote{Note that $\mathcal{L}$ is only required to be $\mu$-strongly convex in the neighborhood of the loss basin including $\bm\theta$ and $\bm\theta^*$, instead of the entire $\mathbb{R}^n$.}, and $D_f(p^*||p)\le C_f$, there exists
\begin{align} 
C_{\bm\delta}=\frac{1+o(1)}{\mu}\sqrt{\frac{C_f}{a_2}\mathbb{E}_{p(\vect{z})}\big[\|\nabla_{\bm\theta}\ell(\bm\theta;\vect{z})\|_2^2\big]}
\end{align}
such that $\|\bm\delta\|_2\le C_{\bm\delta}$, namely $C_{\bm\delta}$ is a upper bound.
\label{propA:norm}
\end{propA}
\begin{proof}

With the change-of-measure technique, we have:
\begin{align}
\mathbb{E}_{p(\vect{z})}\big[r(\vect{z})\big]=0.
\end{align}

Then
\begin{align}
&D_f\left(p^*||p\right)= \mathbb{E}_{p(\vect{z})}\big[f(1+r(\vect{z}))]\\
&= \mathbb{E}_{p(\vect{z})}\big[a_1r+a_2|r|^2+o(r)^2]\\
&=(a_2+o(1))\mathbb{E}_{p(\vect{z})}\big[|r|^2\big]\le C_f.
\end{align}

Therefore,
\begin{align}
\mathbb{E}_{p(\vect{z})}\big[|r|^{2}\big]\le\frac{(1+o(1))C_f}{a_2}.
\end{align}

According to the upper bound of $\mathbb{E}_{p(\vect{z})}\big[|r|^{2}\big]$,
\begin{align}
&\big\|\mathbb{E}_{p(\vect{z})}\big[r(\vect{z})\nabla_{\bm\theta}\ell(\bm\theta;\vect{z})]\big\|_2 \\
&\le \mathbb{E}_{p(\vect{z})}\big[|r(\vect{z})|^2\big]^\frac{1}{2}\mathbb{E}_{p(\vect{z})}\big[\|\nabla_{\bm\theta}\ell(\bm\theta;\vect{z})\|^2]^\frac{1}{2}\\
&\le \sqrt{\frac{(1+o(1))C_f}{a_{2}}\mathbb{E}_{p(\vect{z})}\big[\|\nabla_{\bm\theta}\ell(\bm\theta;\vect{z})\|^2\big]}.
\end{align}

Therefore, 
\begin{align}
&\mu\|\bm \delta\|_2\le
\|\bm H\bm \delta\|_2\\
&=\big\|\mathbb{E}_{p(\vect{z})}\big[r(\vect{z})\nabla_{\bm\theta}\ell(\bm\theta;\vect{z})]\big\|_2\\
&\le \sqrt{\frac{(1+o(1))C_f}{a_{2}}\mathbb{E}_{p(\vect{z})}\big[\|\nabla_{\bm\theta}\ell(\bm\theta;\vect{z})\|^2\big]}.
\end{align}

There exists 
\begin{align}
C_{\bm\delta}=\frac{1+o(1)}{\mu}\sqrt{\frac{C_f}{a_2}\mathbb{E}_{p(\vect{z})}\big[\|\nabla_{\bm\theta}\ell(\bm\theta;\vect{z})\|_2^2\big]}
\end{align}
such that $\|\bm\delta\|_2\le C_{\bm\delta}$.

\end{proof}

\begin{propA}
\label{propA:data}
Suppose the mixed distribution of $\mathcal{D}^\text{mix}$ is $p^\text{mix}=(1-\eta)p+\eta p^*$, then we have $D_f(p^\text{mix}||p)\le C_f^\text{mix}=\eta^2 C_f$. Denote $\bm\theta^\text{mix}$ as the optimal parameter on $\mathcal{D}^\text{mix}$, $\bm\delta^\text{mix}=\bm\theta^\text{mix}-\bm\theta$, then we have:
\begin{align}
\frac{\|\bm\delta^\text{mix}\|_2}{\|\bm\delta\|_2}=\frac{C_{\bm\delta}^\text{mix}}{C_{\bm\delta}}=\eta+o(1).
\end{align}
\end{propA}

\begin{proof}
\begin{align}
r^\text{mix}(\vect{z})=\frac{p^\text{mix}(\vect{z})}{p(\vect{z})}-1=\eta\times  r(\vect{z}).
\end{align}

\end{proof}

\begin{propA}
\label{propA:gradient}
Define the average gradient strength as $G=\mathbb{E}_{p(\vect{z})}[\|\nabla_{\bm\theta}\ell(\bm\theta;\vect{z})\|_2]/\sqrt{n}$, and the average gradient strength of group $i$ as $G_{(i)}=\mathbb\mathbb{E}_{p(\vect{z})}[\|\nabla_{\bm\theta_{(i)}}\ell(\bm\theta;\vect{z})\|_2]/\sqrt{n_{(i)}}$, then 
\begin{align}
\|\bm\delta\|_2\propto \frac{\sqrt{n}}{G},\quad \|\bm\delta_{(i)}\|_2 \propto \frac{\sqrt{n_{(i)}}}{G_{(i)}}.
\end{align}
\end{propA}
\begin{proof}

Suppose $\lambda$ denotes the average values of eigenvalues of the Hessian matrix, according to the Fisher information matrix assumption~\citep{fisher_gradient},
\begin{align}
\bm H&=\mathbb{E}_{p(\vect{z})}[\nabla_{\bm\theta}\ell(\bm\theta;\vect{z})\nabla_{\bm\theta}\ell(\bm\theta;\vect{z})^\text{T}],\\
\lambda&=\frac{\text{tr}(\bm H)}{n}\\
&=\frac{\mathbb{E}_{p(\vect{z})}[\text{tr}(\nabla_{\bm\theta}\ell(\bm\theta;\vect{z})\nabla_{\bm\theta}\ell(\bm\theta;\vect{z})^\text{T})]}{n}\\
&=\frac{\mathbb{E}_{p(\vect{z})}[\|\nabla_{\bm\theta}\ell(\bm\theta;\vect{z})\|_2^2]}{n}\propto G^2.
\end{align}

We have,
\begin{align}
{\bm H}\bm\delta = -\mathbb{E}_{p(\vect{z})}[r(\vect{z})\nabla_{\bm\theta}\ell(\bm\theta;\vect{z})]&+o(\|\bm\delta\|_2),\\
\|{\bm H}\bm\delta\|_2 \propto \lambda\|\bm\delta\|_2 &\propto G^2\|\bm\delta\|_2,\\
\|\mathbb{E}_{p(\vect{z})}[r(\vect{z})\nabla_{\bm\theta}\ell(\bm\theta;\vect{z})]\|_2 &\propto  \sqrt{n}G.
\end{align}

Therefore,
\begin{align}
G^2\|\bm\delta\|_2 &\propto \sqrt{n}G,\\
\|\bm\delta\|_2 &\propto \frac{\sqrt{n}}{G}.
\end{align}

Assume ${\partial^2\mathcal{L}(\bm\theta;\mathcal{D})}/{(\partial\bm\theta_{(i)}\partial\bm\theta_{(j)})}=0$ for $i\ne j$, namely $\bm H=\text{diag}\{\bm H_{(1)}, \bm H_{(2)}, \cdots, \bm H_{(l)}\}$, where $\bm H_{(i)}=\mathbb{E}_{p(\vect{z})}[\nabla_{\bm\theta_{(i)}}\ell(\bm\theta;\vect{z})\nabla_{\bm\theta_{(i)}}\ell(\bm\theta;\vect{z})^\text{T}]$, then
\begin{align}
{\bm H_{(i)}}\bm\delta_{(i)} = -\mathbb{E}_{p(\vect{z})}[r(\vect{z})\nabla_{\bm\theta_{(i)}}\ell]&+o(\|\bm\delta\|_2),\\
\|\bm\delta_{(i)}\|_2 &\propto \frac{\sqrt{n_{(i)}}}{G_{(i)}}.
\end{align}
\end{proof}

\subsection{Details of Multi-step Implementation}

\citet{multi-step-defense} give the close-formed solution of $\vect{u}_{k+1}$ under the constraint $\|\vect{u}\|_p=\eta$,
\begin{align}
\vect{u}_{k+1}&=\argmax_{\|\vect{u}\|_p=\eta}\vect{g}_{k}^\text{T}\vect{u}\\
&=\eta\left(\text{sgn}(\vect{g_{k}})\odot\frac{|\vect{g}_{k}|^\frac{1}{p-1}}{\||\vect{g}_{k}|^\frac{1}{p-1}\|_p}\right).
\end{align}

When the constraint is $S=\{\vect{u}:\|\bm{T}^{-1}\vect{u}\|_p=\eta\}$, $\vect{g}_{k}^\text{T}\vect{u}=(\bm{T}\vect{g}_{k}^\text{T})(\bm{T}^{-1}\vect{u})$, namely
\begin{align}
\bm{T}^{-1}\vect{u}_{k+1}=\frac{\eta\big(\text{sgn}(\bm{T}\vect{g}_k)\big)\odot|{\bm T}\vect{g}_k|^\frac{1}{p-1}}{\||{\bm T}\vect{g}_k|^\frac{1}{p-1}\|_p}.
\end{align}

Therefore,
\begin{align}
\vect{u}_{k+1}=\frac{\eta\big(\bm{T}\text{sgn}(\vect{g}_k)\big)\odot|{\bm T}\vect{g}_k|^\frac{1}{p-1}}{\||{\bm T}\vect{g}_k|^\frac{1}{p-1}\|_p}.
\end{align}

To get the $k$-th corruption $\vect{a}_k$, we project the updated corruption $\vect{a}_{k-1}+\vect{u}_k$ into the set $S$. \citet{multi-step-defense} give the projecting functions, 
\begin{align}
\Pi_S(\vect{v})&=\min\{\|\vect{v}\|_2, \epsilon\}\frac{\vect{v}}{\|\vect{v}\|_2}\quad (p=2);\\ 
\Pi_S(\vect{v})&=\text{clip}(\vect{v}, -\epsilon, \epsilon)\quad (p={+\infty});
\end{align}
similarly, when $S=\{\vect{u}:\|\bm{T}^{-1}\vect{u}\|_p\le\epsilon\}$, we may assume  ${\bm T}^{-1}\Pi_S(\vect{v})=\Pi_S({\bm T}^{-1}\vect{v})$. Suppose $\vect{x}={\bm T}^{-1}\vect{v}$, we have,
\begin{align}
{\bm T}^{-1}\Pi_{\|\vect{x}\|_2\le\epsilon}{(\vect{v})} =& {\min\{\|\vect{x}\|_2,\epsilon\}}\frac{\vect{x}}{\|\vect{x}\|_2};\\
{\bm T}^{-1}\Pi_{\|\vect{x}\|_{+\infty}\le\epsilon}{(\vect{v})} &= \text{clip}(\vect{x},-\epsilon, \epsilon).
\end{align}

Therefore,
\begin{align}
\Pi_{\|{\bm T}^{-1}\vect{v}\|_2\le\epsilon}{(\vect{v})} =& \frac{\min\{\|{\bm T}^{-1}\vect{v}\|_2,\epsilon\}\vect{v}}{\|{\bm T}^{-1}\vect{v}\|_2};\\
\Pi_{\|{\bm T}^{-1}\vect{v}\|_{+\infty}\le\epsilon}{(\vect{v})} &= {\bm T}\text{clip}({\bm T}^{-1}\vect{v},-\epsilon, \epsilon).
\end{align}

\section{Datasets and Baseline Implementations}

In this section, we introduce the datasets and baseline implementations. The models enhanced with SAM algorithms adopt the same hyper-parameters to baselines except for the SAM hyper-parameter settings ($K, T_{(i)}, \epsilon, L_p$). All experiments are conducted on NVIDIA TITAN RTX GPUs. We conduct every experiment for $4$ runs and report the mean and validation results.

\subsection{IMDB}

We implement the base model TextCNN~\cite{TextCNN} on the IMDb movie reviews dataset (\textbf{IMDB})~\citep{IMDB}. The reviews are classified into $3$ classes, namely the positive reviews, negative reviews, and neutral reviews. The training size is $25000$, the validation size is $25000$, and the test size is $50000$. In the preprocessing of the text, the vocab size is $30000$, and the text length is $200$. The embedding size is $500$. In TextCNN, the filter window sizes are $3$, $4$, and $5$, with $500$ feature maps each. The optimizer is Adam with a learning rate of $10^{-3}$ and a batch size of $64$. We train models for $10$ epochs and test the accuracy on the checkpoint with the best valid accuracy. 

\subsection{PTB-LM} 
We implement a 2-layer LSTM as a language model following \citet{merityRegOpt} on the word-level Penn TreeBank dataset (PTB)\footnote{Dataset is available at \url{https://www.kaggle.com/nltkdata/penn-tree-bank?select=ptb}}\citep{PTB-LM}. In the preprocessing of the text, the vocab size is $10000$. We use the SGD optimizer with an initial learning rate of $50$ and a gradient norm clip of $0.25$. We adopt a learning rate decay of $0.5$ after $10$ epochs. The input and output embeddings are tied. The embedding size is $500$, and the hidden size is $500$. We train models for $20$ epochs and test the perplexity on the checkpoint with the lowest valid perplexity. 

\subsection{En-Vi}
We implement the transformer model following the fairseq~\citet{fairseq} ``transformer\_wmt\_en\_de'' implementation as our baseline model on the En-Vi dataset, which is provided by the IWSLT 2015 Evaluation Campaign~\citep{2015iwslt}. The training size is 133K, the validation set is TED tst2012 with a size of $1553$ sentences, and the test set is TED tst2013 with a size of $1268$ sentences. In the preprocessing of the text, the English and Vietnamese vocabulary sizes are $17200$ and $7800$ respectively. We use the same hyper-parameters following fairseq~\citep{fairseq}. We train the model for $52$ epochs and adopt a checkpoint average of $10$. In testing, We adopt the BLEU metric and the beam size of the model inference is $5$.

\subsection{De-En}
We implement the transformer model following the fairseq~\citet{fairseq} ``transformer\_wmt\_en\_de'' implementation as our baseline model on the De-En dataset, which is provided by the IWSLT 2014 Evaluation Campaign~\citep{2014iwslt}. We use the same dataset splits and the same hyper-parameters following fairseq~\citep{fairseq}. The training, validation and test sizes are 153K, 7K, and 7K, respectively. In the preprocessing of the text, we adopt the BPE technique, and the German and English vocabulary sizes are $8848$ and $6632$ respectively. We train the model for $70$ epochs and adopt a checkpoint average of $10$. In testing, We adopt the BLEU metric and the beam size of the model inference is $5$.

\section{Experimental Settings}

In this section, we report experimental settings in the paper, including analytic trial details and hyper-parameters of different algorithms.

\subsection{Analytic Trial Settings}

We conduct a 3-layer Multi-Layer Perceptrons (MLP, sizes are $784, 100, 100, 10$) on the MNIST dataset~\citep{MNIST}. The test minimum $\bm\theta^*$ is fine-tuned from the training minimum $\bm\theta$, and $\bm\delta=\bm\theta^*-\bm\theta$. We utilize linear interpolation to plot $\alpha\bm\theta+(1-\alpha)\bm\theta^*$ for different $\alpha$, for visualizing the training loss, test loss and the shifted training loss. It can be concluded that the shifted training loss is similar to the test loss near the local minimum. To visualize the relation between $\|\bm\delta\|_2$ and the distribution shift. We conduct 100 trials with the same training set $\mathcal{D}$ and different mixed test sets $\mathcal{D}^\text{mix}$ (mixed with $(1-\eta)$ of the training data from $\mathcal{D}$ and $\eta$ of the test data from $\mathcal{D}^*$). Here the test minimum is fine-tuned from the training minimum and $\bm\delta=\bm\theta^*-\bm\theta$. $\eta$ can be utilized to measure the strength of the distribution shift between $\mathcal{D}$ and $\mathcal{D}^\text{mix}$. $\eta=1\%, 2\%, 3\%, \cdots, 99\%, 100\%$. Results show that there exists an approximately linear relationship between $\|\bm\delta\|_2$ and distribution shift strengths.

\subsection{Main Result Settings}

We try both $L_2$ and $L_{+\infty}$. For $\epsilon$, we first binary search proper order of magnitude, for example $10^{-5}$ to $10^{-4}$, then we grid search $\epsilon$, for example, $\{1\times 10^{-5}, 2\times 10^{-5}, 5\times 10^{-5}, 8\times 10^{-5}, 1\times 10^{-4}, 2\times 10^{-4}, 5\times 10^{-4}, 8\times 10^{-4}\}$. 

\textbf{SAM on Transformer Models.}
The Transformer models need larger $K$ and more detailed hyper-parameter tuning since SAM learning on the Transformer is unstable. Besides, \citet{multi-step-defense} propose that SAM learning in the early stage may harm the learning. In our GA-SAM, $\|\vect{g}\|_2$ is large and we can omit the SAM learning in the early stage. Therefore, we adopt $K=2$ in multi-step implementations, and following \citet{multi-step-defense}, we adopt a start epoch of $30$ for Transformer models.

\subsubsection{IMDB}

On the IMDB dataset, we adopt $K=1$. On SAM, $L_{+\infty}, \epsilon=2\times 10^{-5}$. On ASAM, $L_{+\infty}, \epsilon=5\times 10^{-3}$. On Layer-SAM, $L_{2}, \epsilon=1\times 10^{-5}$. On Multi-step Defense, $L_{+\infty}, \epsilon=2\times 10^{-5}$. On GA-SAM, $L_{+\infty}, \epsilon=1\times 10^{-4}$.

\subsubsection{PTB-LM} 

On the PTB-LM dataset, we adopt $K=1$. On SAM, $L_{2}, \epsilon=5\times 10^{-3}$. On ASAM, $L_{+\infty}, \epsilon=5\times 10^{-4}$. On Layer-SAM, $L_{2}, \epsilon=1\times 10^{-6}$. On Multi-step Defense, $L_{+\infty}, \epsilon=1\times 10^{-3}$. On GA-SAM, $L_{+\infty}, \epsilon=8\times 10^{-6}$.

\subsubsection{En-Vi}

On the En-Vi dataset, we adopt $K=2$ in multi-step implementations, and following \citet{multi-step-defense}, we adopt a start epoch of $30$ and similar hyper-parameters. On SAM, $L_{+\infty}, \epsilon=1.2\times 10^{-3}$. On ASAM, $L_{+\infty}, \epsilon=5\times 10^{-4}$. On Layer-SAM, $L_{+\infty}, \epsilon=1\times 10^{-3}$. On Multi-step Defense, $L_{+\infty}, \epsilon=5\times 10^{-4}$. On GA-SAM, $L_{+\infty}, \epsilon=0.9$.

\subsubsection{De-En}

On the De-En dataset, we adopt $K=2$ in multi-step implementations, and following \citet{multi-step-defense}, we adopt a start epoch of $30$ and similar hyper-parameters. On SAM, $L_{+\infty}, \epsilon=5\times 10^{-4}$. On ASAM, $L_{+\infty}, \epsilon=2\times 10^{-4}$. On Layer-SAM, $L_{+\infty}, \epsilon=3\times 10^{-3}$. On Multi-step Defense, $L_{+\infty}, \epsilon=5\times 10^{-4}$. On GA-SAM, $L_{+\infty}, \epsilon=0.8$.

\subsection{Ablation Study Settings}

In this section, we report hyper-parameter settings in the ablation study.

\subsubsection{IMDB}

On the IMDB dataset. On single-step implementation, $L_{+\infty}, \epsilon=5\times 10^{-5}$. On element-wise implementation, $L_{+\infty}, \epsilon=5\times 10^{-6}$. On model-wise implementation, $L_{+\infty}, \epsilon=5\times 10^{-5}$. For $T_{(i)}={\|\vect{w}_{(i)}\|_2}/{\|\vect{g}_{(i)}\|_2}$, $L_{+\infty}, \epsilon=2\times 10^{-5}$. For $T_{(i)}={1}/{\|\vect{g}_{(i)}\|_2}$, $L_{+\infty}, \epsilon=5\times 10^{-5}$. For $T_{(i)}={\|\vect{w}_{(i)}\|_2}/{\sqrt{n_{(i)}}}$, $L_{+\infty}, \epsilon=5\times 10^{-3}$. For $T_{(i)}={\|\vect{w}_{(i)}\|_2}$, $L_{+\infty}, \epsilon=5\times 10^{-6}$.

\subsubsection{PTB-LM} 

On the PTB-LM dataset. On single-step implementation, $L_{2}, \epsilon=5\times 10^{-3}$. On element-wise implementation, $L_{+\infty}, \epsilon=2\times 10^{-6}$. On model-wise implementation, $L_{+\infty}, \epsilon=5\times 10^{-6}$. For $T_{(i)}={\|\vect{w}_{(i)}\|_2}/{\|\vect{g}_{(i)}\|_2}$, $L_{2}, \epsilon=2\times 10^{-5}$. For $T_{(i)}={1}/{\|\vect{g}_{(i)}\|_2}$, $L_{+\infty}, \epsilon=5\times 10^{-6}$. For $T_{(i)}={\|\vect{w}_{(i)}\|_2}/{\sqrt{n_{(i)}}}$, $L_{+\infty}, \epsilon=1\times 10^{-4}$. For $T_{(i)}={\|\vect{w}_{(i)}\|_2}$, $L_{+\infty}, \epsilon=1\times 10^{-7}$.

\subsection{Adversarial Training Settings}

In this section, we report hyper-parameter settings in the study of adversarial training.

\subsubsection{En-Vi}

On the En-Vi dataset, we adopt $K=1$ and adopt a start epoch of $30$. We adopt the $L_{+\infty}$ constraint on virtual attacks on word embeddings and grid search $\epsilon$ in $\{$ 0.001, 0.002, 0.005, 0.01, 0.02, 0.05, 0.1, 0.2, 0.5 $\}$. The best configuration is $\epsilon=0.02$.

\subsubsection{De-En}

On the De-En dataset, we adopt $K=3$ and adopt a start epoch of $30$. We adopt the $L_{+\infty}$ constraint on virtual attacks on word embeddings  and grid search $\epsilon$ in $\{$ 0.001, 0.002, 0.005, 0.01, 0.02, 0.05, 0.1, 0.2, 0.5 $\}$. The best configuration is $\epsilon=0.005$.

\end{document}